\documentclass{article}

\usepackage[accepted]{icml2025}   

\usepackage[utf8]{inputenc}     
\usepackage{graphicx}
\usepackage{subcaption}           
\usepackage{microtype}
\usepackage{booktabs}
\usepackage{amsmath,amssymb,amsthm}
\usepackage{enumitem}
\usepackage{textgreek}
\usepackage{xcolor}
\usepackage{listings}
\usepackage{lmodern}
\usepackage{balance}
\usepackage{fancyhdr}
\fancypagestyle{firstpage}{\fancyhf{}}
\newcommand{\WS}{\mathrm{WILSON}}
\newcommand{\UCCT}{\mathrm{UCCT}}

\theoremstyle{plain}
\newtheorem{theorem}{Theorem}[section]

\newtheorem{lemma}[theorem]{Lemma}
\newtheorem{corollary}[theorem]{Corollary}
\theoremstyle{definition}

\theoremstyle{remark}

\newcommand{\Npos}{6}          
\newcommand{\Nlayers}{5}       
\newcommand{\DX}{0.9}          
\newcommand{\DY}{0.6}          
\newcommand{\LoopIStart}{2}
\newcommand{\LoopIEnd}{3}
\newcommand{\LoopLStart}{1}
\newcommand{\LoopLEnd}{2}
\newcommand{\DotSize}{1.3pt}
\newcommand{\AttnLift}{0.35}   
\usepackage{hyperref}  

\usepackage{refcount}

\newcommand{\figref}[1]{Fig.~\getrefnumber{#1}}

\usepackage{tikz}
\usetikzlibrary{arrows.meta,positioning,calc,decorations.pathreplacing,fit}
\usepackage{pgfplots}
\pgfplotsset{compat=1.18}
\usepgfplotslibrary{groupplots} 

\usepackage{tikz-cd}

\icmltitlerunning{Inverse-Free Wilson Loops for Transformers}

\makeatletter
\@ifundefined{icml@copyrighthdr}{}{%
  \renewcommand{\icml@copyrighthdr}{}%
}
\@ifundefined{icml@copyrighttext}{}{%
  \renewcommand{\icml@copyrighttext}{}%
}
\@ifundefined{icmlfancyfoot}{}{%
  \renewcommand{\icmlfancyfoot}[1]{}%
}
\icmltitlerunning{Inverse-Free Wilson Loops for Transformers}

\begin{document}
    
\twocolumn[
\icmltitle{Inverse-Free Wilson Loops for Transformers: A Practical Diagnostic for Invariance and Order Sensitivity}

\begin{icmlauthorlist}
\icmlauthor{Edward Y. Chang, Stanford University}{stan}
\icmlauthor{Ethan Y. Chang, UIUC}{uiuc}
\end{icmlauthorlist}

\icmlkeywords{Large Language Models, Group Theory, Curvature, Diagnostics}

\vskip 0.3in  
]


\begin{abstract}
Large language models can change answers under harmless edits that matter in practice: RAG outputs flip when passages are reordered, fine-tuning erodes invariances learned at pretraining, debate or chain-of-thought prompts take path-dependent routes, and compiler fusion or reordering perturbs logits near decision boundaries. These failures violate intended invariances, break continuous integration, and force teams to trade safety for speed. The effects are small yet distributed across layers and positions, sensitive to context length and evaluation order, and costly to repair with retraining or formal verification.
We present $\WS$, a minimal post-hoc diagnostic suite that converts simple loop and reordering checks on internal representations into system signals. $\WS$ combines an inverse-free curvature map over positions and layers, computed with JVPs and Hutchinson probes, with activation-level commutators that flag reorder risk. Signals are cheap to compute, model-agnostic for standard Transformers, and exported as thresholds and CSV artifacts for orchestrators. This enables concrete actions: guard RAG against order effects, catch fine-tuning regressions, stabilize debate pathways and long multi-turn contexts, and gate fusions or reorders in deployment. In short, WILSON helps anticipate failures and approve safe optimizations so reliability and throughput can improve together without changing model architecture or training.
\end{abstract}

\section{Introduction}
\label{sec:intro}

Modern LLM deployments need diagnostics that (i) anticipate when model behavior will violate intended \emph{invariances} (e.g., code $\alpha$-renaming; algebraic equivalence), and (ii) tell systems when it is safe to \emph{fuse, reorder, or parallelize} operators without changing outputs. We introduce $\WS$, a minimal post-hoc diagnostic suite that converts simple loop and reordering checks on internal representations into system signals. We call the suite $\WS$ in reference to Wilson loops that measure holonomy; the name is a mnemonic rather than an acronym.

\paragraph{Why this matters now.}
LLMs are embedded in pipelines where \emph{behavioral stability} is a hard requirement: builds must pass CI, compilers seek fusions or reorders for latency, and products promise consistent outcomes across benign edits. In practice we see: (1) \textbf{invariance breaks} under semantics-preserving transforms such as identifier renaming, algebraic equivalence, and template-preserving paraphrase (Fig.~\ref{fig:alpha-rename}); (2) \textbf{order sensitivity}, where seemingly safe operator changes alter outputs, exacerbated by floating-point non-associativity and long-context phase drift; and (3) \textbf{seed instability} that complicates interpretability and audits. These issues surface in real applications: RAG answers can flip when retrieved passages are reordered, fine-tuning can erode invariances learned at pretraining, and debate or chain-of-thought prompts can take path-dependent routes to different conclusions. Such failures raise costs (canaries, rollbacks), obscure root causes, and block safe speedups.

\paragraph{Why not just train it away?}
Architectural equivariance and invariance-aware training help, but they (a) constrain model classes, (b) require retraining, and (c) still leave open \emph{online} decisions about where fusion or reordering is safe for the \emph{deployed} checkpoint. Saliency or probing can be informative but are gauge-unstable across seeds. Mechanistic case studies are insightful but hard to scale. Formal verification is too heavy for routine inference. We instead seek \emph{post-hoc, cheap, gauge-stable} signals that travel with any Transformer and slot into existing systems.

\subsection*{Two motivating examples (why a diagnostic is needed)}
\label{sec:examples}

\paragraph{Example 1: $\alpha$-renaming should not change behavior (but often does).}
The two programs below are \emph{semantically identical}. A robust code LLM should produce the same decision (e.g., \texttt{pass@k}) and similar next-token probabilities under either identifier scheme. An \textbf{invariance failure} occurs when the invariance ratio $\mathrm{IR}(x) < 1 - \delta$ (Sec.~\ref{sec:metrics}).

\begin{lstlisting}[
  language=Python, label={lst:alpha-rename},
  basicstyle=\ttfamily\footnotesize,
  breaklines=true, breakatwhitespace=true,
  columns=fullflexible, keepspaces=true,
  postbreak=\mbox{\textcolor{gray}{$\hookrightarrow$}\space},
  caption={Alpha-renaming: semantics preserved, identifiers changed.}
]
# (A) Original
def is_balanced(s: str) -> bool:
    stk = []
    for ch in s:
        if ch == '(':
            stk.append(ch)
        elif ch == ')':
            if not stk: return False
            stk.pop()
    return not stk

# (B) Renamed (alpha-equivalent)
def is_balanced(u: str) -> bool:
    Z = []
    for c in u:
        if c == '(':
            Z.append(c)
        elif c == ')':
            if not Z: return False
            Z.pop()
    return not Z
\end{lstlisting}

\emph{What goes wrong.} Identifier statistics from pretraining can shift tokenization and local logits enough to change \texttt{pass@k} or a critical beam step. Our diagnostic goal is to predict such failures \emph{before} they happen. High curvature $\kappa_{\mathrm{inv}}$ near the relevant positions and layers flags risk, and elevated commutators $\Delta_{A,B}$ identify order-sensitive submodules that amplify the effect.

\paragraph{Example 2: ``Safe'' operator changes can flip the next token (order sensitivity).}
Even mathematically equivalent rewrites can differ numerically in finite precision. Consider attention logits $z=\tfrac{QK^\top}{\sqrt d}+M$. Two fused kernels compute $z$ with different accumulation orders:
\begin{footnotesize}
\[
\mathrm{Path\ A:}\;(\sum_i q_i k_i)/\sqrt d + m \quad \text{vs.} \quad
\mathrm{Path\ B:}\;\sum_i (q_i k_i/\sqrt d) + m,
\]
\end{footnotesize}
equal in $\mathbb{R}$ but \emph{not} in FP16 or BF16 due to rounding. A perturbation on the order of $10^{-3}$ can flip the top token.

\begin{lstlisting}[
  language=Python, label={lst:tie-flip},
  basicstyle=\ttfamily\footnotesize,
  breaklines=true, breakatwhitespace=true,
  columns=fullflexible, keepspaces=true,
  postbreak=\mbox{\textcolor{gray}{$\hookrightarrow$}\space},
  caption={Toy numeric: tiny logit changes at a tie boundary can flip the argmax.}
]
import numpy as np

# Near-tie logits at step t
zA = np.array([0.000,  0.001], dtype=np.float32)   # Path A
zB = np.array([-0.001, 0.001], dtype=np.float32)   # Path B

pA = np.exp(zA) / np.exp(zA).sum()
pB = np.exp(zB) / np.exp(zB).sum()

# Nudge by 1e-3 (within fused-kernel rounding)
zB_pert = zB + np.array([+0.0015, -0.0005], dtype=np.float32)
pB_pert = np.exp(zB_pert) / np.exp(zB_pert).sum()
# The argmax can flip when margins are razor thin.
\end{lstlisting}

\emph{What to measure.} We define \textbf{drift} $\delta(x)=\|y_{AB}-y_{BA}\|$ between outputs under two safe reorder or fuse options and use the \emph{commutator} $\Delta_{A,B}=\|A\!\circ\!B - B\!\circ\!A\|_F$ as a predictor of this drift. In parallel, the \emph{inverse-free curvature} $\kappa_{\mathrm{inv}}$ (Eq.~\eqref{eq:invfree-kappa}) localizes \emph{where} order sensitivity is likely (high-curvature loops), allowing systems to avoid reorders or fusions there or to add checks.

\paragraph{Our route in one line.}
Treat the model as a discrete bundle over (position, layer); form small plaquette loops from vertical layer transports and horizontal attention transports; estimate \emph{inverse-free} holonomy $\kappa_{\mathrm{inv}}$ with $r$ Hutchinson JVP probes per loop to localize order sensitivity; compute activation commutators $\Delta_{A,B}$ on a calibration batch to predict reorder risk for candidate submodules; apply an $O(d)$ orthogonal gauge fix (whitening plus Procrustes) to 
align logs across seeds;
compute $\kappa_{\mathrm{inv}}$ in the native basis; and emit curvature
and commutator scores as CSV with calibrated thresholds that gate fusions, reorders, and parallel routes in CI.

\textit{Terminology note.} We borrow terms such as holonomy and commutator for intuition, but we use discrete, data-dependent path differences on Transformer representations as engineering diagnostics rather than claims about continuous field theories.

\paragraph{LLM-specific applications.}
While $\WS$ is model-agnostic, several high-confidence use cases emerge for LLM systems: prompt robustness under paraphrasing, RAG passage reordering sensitivity, fine-tuning regression detection, multi-turn drift monitoring, and chain-of-thought pathway stability (\S6.8). These applications require only diagnostic signals with no retraining. We also identify potential synergies with semantic anchoring frameworks such as $\UCCT$. Empirical validation of these connections remains future work (\S8.2).

\paragraph{How this helps semantic anchoring ($\UCCT$).}
Semantic anchoring \cite{chang2025UCCT} benefits from geometry-derived signals. Prefer low-$\kappa_{\mathrm{inv}}$ regions for anchor placement (flat loops), gate anchors with high local commutators, and maintain anchor health via gauge-fixed logs. We treat this as an application of our diagnostics (details in \S\ref{sec:systems}); nothing in our method depends on $\UCCT$.

\paragraph{Design goals.}
We seek diagnostics that are (G1) \textbf{predictive} of failures before they happen, (G2) \textbf{cheap} (JVP-only; matrix-free), (G3) \textbf{gauge-stable} across seeds and runs, (G4) \textbf{model-agnostic} for Transformers with standard residual blocks and RoPE or relative positions, and (G5) \textbf{systems-usable} with clear thresholds and CSV artifacts that feed orchestrators and planners.

\paragraph{Idea in brief.}
We view the network as a \emph{discrete bundle} over \emph{(position, layer)}. Vertical edges carry layer transports. Horizontal edges carry attention transports. We define an \textbf{inverse-free Wilson loop} that measures \emph{holonomy} using only JVPs with Hutchinson probes, yielding a curvature score $\kappa_{\mathrm{inv}}$ (Eq.~\eqref{eq:invfree-kappa}). Large curvature indicates \emph{non-commuting transports} and flags order-sensitive regions. Small curvature suggests safe fusions or reorders. In parallel, we compute \textbf{activation commutators} $\Delta_{A,B}$ between submodules to map order-sensitive pairs and relate them to output drift (Fig.~\ref{fig:diagnostic-suite}(b)). A light \textbf{orthogonal gauge fix} (whitening plus Procrustes; Fig.~\ref{fig:gauge-pipeline}) stabilizes features and logging across seeds in CI.

\paragraph{Contributions.}
\begin{itemize}[leftmargin=1.2em,itemsep=2pt]
\item \textbf{Minimal formalism.} Discrete transports on (position, layer) with $O(d)$ gauge. An \emph{inverse-free} curvature $\kappa_{\mathrm{inv}}$ computable by JVPs. Estimator error derivations and an empirical study of gauge behavior (Appx.).
\item \textbf{Actionable diagnostics.} Curvature scores for predicting invariance failures (ROC and AP; Fig.~\ref{fig:diagnostic-suite}(a)). Commutators for predicting drift under reorder or fuse (Fig.~\ref{fig:diagnostic-suite}(b)). Both expose safe \emph{parallel} regions and risky \emph{sequential} ones.
\item \textbf{Gauge-stable logging.} Whitening plus Procrustes reduces probe variance and stabilizes saliency ranks across seeds, improving CI reproducibility (Fig.~\ref{fig:gauge-pipeline}).
\item \textbf{Systems hooks.} CSV artifacts and thresholds integrate with an \emph{LLM-driven TestBench} and a \emph{CI or release gate} to prioritize metamorphic tests (e.g., $\alpha$-renaming orbits), schedule reorder or fuse canaries, and enforce acceptance thresholds (\S\ref{sec:systems}).
\item \textbf{Reproducible suite.} E1--E7 diagnostics with schemas, defaults, ablations, and a Colab recipe. Outputs are gauge-stable and append-only for forensic replay.
\end{itemize}

\paragraph{Impacts we target.}
(1) \textbf{Reliability}: detect invariance breaks early. (2) \textbf{Throughput}: safe fusion, reordering, and parallelization in low-curvature regions. (3) \textbf{Reproducibility}: gauge-stable logs for CI and audits. (4) \textbf{Cost}: diagnostics with at most 10--20\% overhead and concrete budget knobs.

\paragraph{Metrics (how we judge success).}
We evaluate six axes (\S\ref{sec:metrics}): (M1) \emph{IR} (invariance ratio) on orbits (Fig.~\ref{fig:alpha-rename}); (M2) \emph{AUC or AP} for curvature predicting failures (Fig.~\ref{fig:diagnostic-suite}(a)); (M3) \emph{commutator to drift correlation} (Fig.~\ref{fig:diagnostic-suite}(b)); (M4) \emph{gauge-stable CI} (variance reduction and rank stability); (M5) \emph{RoPE phase drift} versus depth and context length (Fig.~\ref{fig:rope-rotation}); (M6) \emph{overhead} (Fig.~\ref{fig:diagnostic-suite}(c)).

\paragraph{Methods (one-paragraph sketch).}
Vertical transports $T^{\text{layer}}_{i,\ell}\!\approx\!\partial h_{i,\ell+1}/\partial h_{i,\ell}$ and edge-wise horizontal transports $T^{\text{attn}}_{i\leftarrow j,\ell}\!\approx\!\partial h^{\text{out}}_{i,\ell}/\partial h^{\text{in}}_{j,\ell}$ define small loops on the grid. We avoid inverses by comparing two paths that end in the same fiber (Eq.~\eqref{eq:invfree-kappa}), estimating norms with $r$ Rademacher probes via JVPs. A frozen-softmax scan proposes hotspots. We confirm them with full JVPs. Activation commutators measure $\|A\!\circ\!B - B\!\circ\!A\|_F$ on a calibration batch and correlate with output drift under safe reorder or fuse. Gauge fix (whiten plus Procrustes) precedes logging, not curvature.

\paragraph{Experimental setup (overview).}
We will test small or medium open models (7B--13B Transformers with RoPE or relative positions) on code orbits (HumanEval or MBPP style $\alpha$-renaming), algebraic rewrites, synthetic RoPE phase shifts, and calibration batches for commutators. Hardware: A100-class GPUs, BF16 or FP16 with FP32 pins for LN and softmax JVPs. Seeds fixed. Deterministic kernels where available. $(r,k,m)$ defaults $(6,8,6)$. Top-$m$ neighbors by attention mass plus light random exploration. Full procedures appear in \S\ref{sec:setup}. Metrics in \S\ref{sec:metrics}.

\paragraph{Evaluation plan (pre-registered; no results in this version).}
We will report: (i) IR on orbits, (ii) ROC and PR AUC for curvature predicting invariance failures, (iii) Spearman $\rho$ and Pearson $r$ between commutator norms and reorder drift, (iv) gauge-stability deltas (variance ratios and Kendall-$\tau$), (v) RoPE phase-drift areas versus depth and context length, and (vi) end-to-end overhead with attribution (scan versus confirm). Null baselines (randomized curvature maps and random-init networks) and bootstrap CIs are included. Empirical values are intentionally omitted in this version and will appear in a subsequent update.

\paragraph{Scope and non-claims.}
This is a \emph{code-first} note on Transformer residual streams. We do not claim a continuous principal bundle for full models, nor that embeddings realize formal grammar groups. We position curvature and commutator signals as \emph{engineering diagnostics}, leaving broader theory and non-Transformer architectures to future work (\S\ref{sec:conclusion}, Appx.).
\section{Related Work}
\label{sec:related}

\paragraph{Missing invariance tests, fragile behavior, and reorder risk.}
Behavioral test suites show that modern NLP systems often fail simple invariances and controlled perturbations, motivating \emph{diagnostic} evaluation beyond accuracy (e.g., CheckList) \citep{ribeiro2020checklist}. Robustness Gym systematizes stress tests and highlights brittleness under template-preserving edits \citep{goel2021robustgym}. For code LLMs, semantics-preserving program transformations (notably identifier or variable renaming) can spuriously change predictions, motivating invariance-oriented evaluation and augmentation \citep{ankner2021varrename, wang2022recoderobustnessevaluationcode}. At the systems layer, ML compilers (TVM, XLA) aggressively fuse and reorder operators to improve throughput, but floating point non-associativity means reordering can alter numerics \citep{chen2018tvm, higham2002accuracy}. Without guards, such optimizations risk correctness regressions. For long-context LLMs, input order effects are well documented: models can underuse middle context and flip answers when premise or evidence order changes \citep{liu2023lostmiddle, chen2024premiseorder}. In retrieval-augmented pipelines, document order can interact with position bias; recent studies report mixed magnitudes \citep{cuconasu2025rags, zhang2024compensateposbias}.

\paragraph{Attempts to address the problems.}
Architectural symmetry has a long history: group-equivariant models \citep{cohen2016groupcnn, cohen2017steerable, bronstein2021geometric} \emph{bake in} invariances via layer design. Our approach is \emph{post hoc} and model-agnostic: we diagnose order sensitivity and invariance breaks at inference time. For long-context stability, positional schemes such as ALiBi, positional interpolation for RoPE, LongRoPE, and YaRN reduce phase or length drift but modify training or attention parameterization rather than \emph{diagnose} order sensitivity online \citep{press2021alibi, chen2023positioninterp, ding2024longrope, peng2024yarn}. In code robustness, RECode and related work use semantics-preserving transforms and augmentation to improve invariance but again focus on training-time fixes \citep{wang2022recoderobustnessevaluationcode}.

\paragraph{Evaluation frameworks.}
Beyond accuracy, comprehensive evaluations emphasize robustness, calibration, and efficiency, aligning with our diagnostic goals. HELM codifies multi-metric, scenario-based evaluation for LLMs \citep{liang2022helm}. Dynabench advances dynamic, human-in-the-loop stress testing \citep{kiela2021dynabench}. These motivate systematic invariance and order-sensitivity checks alongside standard metrics.

\paragraph{Alternatives and complements: representation alignment and probing.}
Representation-similarity measures (SVCCA, CKA) and orthogonal Procrustes alignment compare layers or seeds while acknowledging gauge freedom, motivating our $O(d)$ analysis gauge and gauge-invariant summaries \citep{raghu2017svcca, kornblith2019cka}. Probing critiques emphasize instability and confounds, arguing for controls and reproducibility, gaps our gauge-fixed logging aims to reduce \citep{hewitt2019designing}. Mechanistic interpretability develops circuit-level tools such as activation patching and progress measures; our commutator and holonomy signals provide a complementary, automatable geometry that flags \emph{order sensitivity} without manual circuit discovery \citep{elhage2021mathematical, nanda2023progress, zhang2024towards}.

\paragraph{Geometry and curvature in learning.}
Discrete curvature notions such as Ollivier and Forman curvature have characterized graphs and GNN dynamics, where curvature reflects structure and flow \citep{samal2018forman, ni2015ricci}. We differ by defining \emph{inverse-free, data-dependent holonomy on Transformer representations} using JVPs, turning curvature into a practical, per-layer and per-position diagnostic tied to invariance and operator order. Unlike graph-level discrete Ricci approaches \citep{samal2018forman, ollivier2007markov}, we compute inverse-free, data-dependent holonomy on internal representations via JVPs to yield granular diagnostics.

\paragraph{ML systems: fusion, scheduling, and safety.}
Compiler stacks (TVM, XLA) and inference schedulers search fusions or reorderings for latency and throughput \citep{chen2018tvm}. Floating point reordering risk and nondeterminism motivate \emph{gates} that decide when such transformations are safe, which is where our curvature and commutator signals act as low-overhead guards. Deterministic summation methods show how fixed accumulation orders can reduce variance across runs and hardware \citep{ahrens2020reprod}. Our signals also align with orchestration frameworks (e.g., the UCCT line of work on semantic anchoring and tool coordination \cite{chang2025UCCT}), where curvature can prioritize \emph{parallel} low-risk regions and enforce invariants during \emph{sequential} high-risk steps.

\paragraph{Physics foundations: groups, gauges, bundles, holonomy.}
Our geometric lens borrows classical ideas from mathematical and theoretical physics. Group symmetries and Noether’s theorem connect invariants to conserved quantities \citep{weyl1952symmetry, noether1918invariante}. Gauge theory formalizes local symmetry with connections and parallel transport on fiber bundles \citep{yang1954gauge, kobayashi1963foundations, nakahara2003gtp, frankel2011geometry}. Holonomy and Wilson loops quantify curvature via loop transports \citep{ambrose1953holonomy, wilson1974confinement}; geometric phases offer an operational view in quantum mechanics \citep{simon1983holonomy, berry1984phase}. Unlike gauge or group equivariant architectures that \emph{instantiate} symmetry in layers, we apply these notions \emph{post hoc} to produce model-agnostic, gauge-stable diagnostics for Transformers.

\section{Method}
\label{sec:method}

\paragraph{Notation.}
Let $B=\{1..T\}\times\{0..L\}$ denote the discrete base (token position, layer) and $F=\mathbb{R}^d$ the fiber (residual stream). A \emph{vertical} transport $T^{\mathrm{layer}}_{i,\ell}\!:\!F\!\to\!F$ maps $(i,\ell)\!\to\!(i,\ell{+}1)$; a \emph{horizontal} transport $T^{\mathrm{attn}}_{i\leftarrow j,\ell}\!:\!F\!\to\!F$ maps $(j,\ell)\!\to\!(i,\ell)$. We analyze in the \emph{orthogonal gauge} ($O(d)$). Unless noted, models are decoder-only Transformers with pre-LN blocks and RoPE/relative positions. Sampling knobs: probes $r$, targets $k$/layer, neighbors $m$/target.

This section gives a code-first recipe for computing: (i) permutation/positional checks, (ii) gauge-aware summaries of hidden states, (iii) parameter-space symmetries, (iv) task-orbit invariance, and (v) two geometry signals—\emph{activation commutators} and an \emph{inverse-free} holonomy score via JVPs. Diagnostic visualizations (heatmaps, curvature maps, ROC curves) appear in \S\ref{sec:diagnostics}.

\subsection{Permutation equivariance in self-attention: scope and limits}
\label{sec:method-permutation}
\paragraph{Derivation (no positions/masks).}
With $Q=XW_Q,\ K=XW_K,\ V=XW_V$ and row-softmax $\sigma$, $\mathrm{Attn}(X)=\sigma(QK^\top/\sqrt d)\,V$. For any permutation matrix $P$,
\[
PX\mapsto (PQ,PK,PV)\quad\text{and}
\]
\[
\sigma\!\Big(P\frac{QK^\top}{\sqrt d}P^\top\Big)=P\,\sigma\!\Big(\tfrac{QK^\top}{\sqrt d}\Big)P^\top,
\]
hence $\mathrm{Attn}(PX)=P\,\mathrm{Attn}(X)$ when \emph{no} positional signals/masks are present.

\paragraph{How symmetry breaks in practice.}
(i) Absolute/learned positions do not commute with $P$; (ii) causal/segment/padding masks select a non-permutation-invariant subgraph; (iii) variable length/padding induces row-dependent normalization.

\paragraph{Diagnostic D1.}
\emph{No-positions check}: $\epsilon_{\pi}=\|\mathrm{Attn}(PX)-P\mathrm{Attn}(X)\|_F/\|\mathrm{Attn}(X)\|_F$ over random $\pi$.  
\emph{Mask curve}: sweep context length under causal masks and plot $\epsilon_{\pi}(n)$ (plots in \S\ref{sec:diagnostics}).

\begin{figure*}[th]
\centering
\begin{tikzpicture}[node distance=8mm, >=Latex]
\tikzstyle{tok}=[draw, rounded corners, minimum width=9mm, minimum height=5mm, fill=cyan!10]
\tikzstyle{blk}=[draw, rounded corners, fill=blue!18, inner sep=3pt]
\node[tok] (x1) {x$_1$}; \node[tok, right=6mm of x1] (x2) {x$_2$}; \node[tok, right=6mm of x2] (x3) {x$_3$};
\node[blk, below=10mm of $(x1)!0.5!(x3)$] (attn) {Self-Attention (no positions/masks)};
\node[tok, below=10mm of attn, xshift=-14mm] (y1) {$y_1$}; \node[tok, right=6mm of y1] (y2) {$y_2$}; \node[tok, right=3mm of y2] (y3) {$y_3$};

\draw[->] (x1) -- (attn); \draw[->] (x2) -- (attn); \draw[->] (x3) -- (attn);
\draw[->] (attn) -- (y1); \draw[->] (attn) -- (y2); \draw[->] (attn) -- (y3);

\node[blk, right=28mm of attn, yshift=5mm] (perm) {Permutation $\pi$};
\node[tok, above=3mm of perm, xshift=-14mm] (xp1) {x$_{\pi(1)}$};
\node[tok, right=6mm of xp1] (xp2) {x$_{\pi(2)}$};
\node[tok, right=6mm of xp2] (xp3) {x$_{\pi(3)}$};
\node[blk, below=5mm of perm] (attn2) {Self-Attention (same weights)};
\node[tok, below=5mm of attn2, xshift=-14mm] (yp1) {$y_{\pi(1)}$};
\node[tok, right=6mm of yp1] (yp2) {$y_{\pi(2)}$};
\node[tok, right=6mm of yp2] (yp3) {$y_{\pi(3)}$};

\draw[->] (xp1) -- (attn2); \draw[->] (xp2) -- (attn2); \draw[->] (xp3) -- (attn2);
\draw[->] (attn2) -- (yp1); \draw[->] (attn2) -- (yp2); \draw[->] (attn2) -- (yp3);

\draw[->, dashed] ($(x2.east)+(1mm,0)$) to[bend left=15] node[above]{reorder} ($(xp2.west)+(-1mm,0)$);
\draw[->, dashed] ($(y2.east)+(1mm,0)$) to[bend left=-15] node[below]{same reorder} ($(yp2.west)+(-1mm,0)$);
\node[above=1mm of x2] {$\mathrm{Attn}(\pi X)=\pi\,\mathrm{Attn}(X)$};
\end{tikzpicture}
\caption{Permutation equivariance holds only without positions/masks; outputs permute with inputs.}
\label{fig:perm-eq}
\end{figure*}
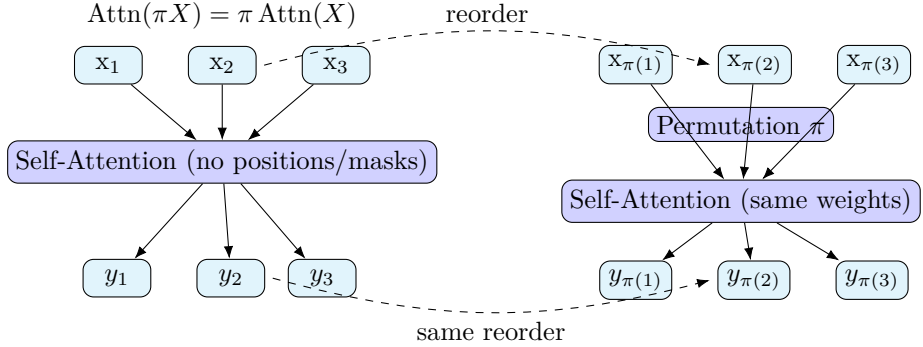

\subsection{Positional structure: sinusoidal, relative, and RoPE}
\label{sec:method-positional}
\paragraph{Absolute sinusoidals.}
A global shift by $\Delta$ induces per-frequency planar rotations on coordinates (block $2{\times}2$ $SO(2)$), stabilizing relative dot-products but \emph{not} enforcing full translation-equivariance.

\paragraph{Relative encodings and RoPE.}
RoPE rotates $Q,K$ so dot-products depend on \emph{phase differences}, controlling relative structure inside attention.

\begin{figure}[th]
\centering
\begin{tikzpicture}[>=Latex, scale=1]
\draw[->] (-2.5,0) -- (2.5,0) node[right] {$q_{2k}$};
\draw[->] (0,-0.2) -- (0,2.2) node[above] {$q_{2k+1}$};
\draw[->, thick] (0,0) -- (1.5,0.6) node[midway, above] {$q$};
\draw[->, thick, color=red] (0,0) -- ({1.5*cos(35) - 0.6*sin(35)},{1.5*sin(35)+0.6*cos(35)})
 node[near end, right, color=red] {$R_\theta q$};
\draw (1.1,0.44) arc[start angle=21, end angle=56, radius=1.2];
\node at (0.9,0.95) {$\theta$};
\end{tikzpicture}
\caption{RoPE applies blockwise planar rotations to query/key coordinates per frequency; dot-products depend on relative phase.}
\label{fig:rope-rotation}
\end{figure}
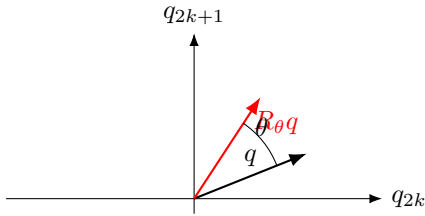

\paragraph{Diagnostic D2 (RoPE drift).}
Apply small phase offsets $\delta$ and track the drift of attention-score distributions vs.\ depth/context length; summarize by area-under-drift (plots in \S\ref{sec:diagnostics}).

\subsection{Rotations and O(d) gauges in representation space}
\label{sec:method-gauge}
\paragraph{Orthogonal gauge and invariants.}
Hidden states are identifiable up to $R\!\in\!O(d)$; report gauge-invariant summaries (norms, pairwise cosines, principal angles).

\paragraph{Gauge-fixing pipeline.}
Whiten per layer and Procrustes-align across seeds; thereafter, log only gauge-invariant features.

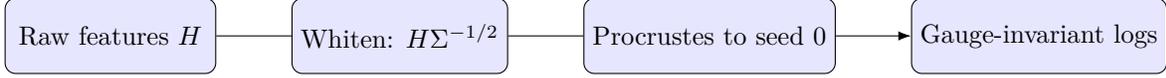
\begin{figure*}[t]
\centering
\begin{tikzpicture}[node distance=5mm, >=Latex]
\tikzstyle{blk}=[draw, rounded corners, minimum width=28mm, minimum height=10mm, fill=blue!10]
\node[blk] (raw) {Raw features $H$};
\node[blk, right=10mm of raw] (white) {Whiten: $H\Sigma^{-1/2}$};
\node[blk, right=10mm of white] (proc) {Procrustes to seed 0};
\node[blk, right=10mm of proc] (log) {Gauge-invariant logs};

\draw[->] (raw) -- (white) -- (proc) -- (log);
\end{tikzpicture}
\caption{Gauge-fixing for reproducible analysis: whiten then align, log gauge-invariant quantities.}
\label{fig:gauge-pipeline}
\end{figure*}

\subsection{Parameter-space symmetries}
\label{sec:method-paramsym}
\paragraph{MLP permutations.}
Compensated hidden-unit permutations preserve $f(x)$ in the idealized setting; we use them as sanity checks for gauge-invariant logging.

\begin{figure}[th]
\centering
\begin{tikzpicture}[node distance=8mm, >=Latex]
\tikzstyle{mat}=[draw, rounded corners, fill=blue!6, inner sep=3pt]
\tikzstyle{arr}=[-Latex, thick]

\node[mat] (x) {$x$};
\node[mat, right=10mm of x] (W1) {$W_1$};
\node[mat, right=10mm of W1] (phi) {$\phi$};
\node[mat, right=10mm of phi] (W2) {$W_2$};
\node[mat, right=10mm of W2] (y) {$y$};
\draw[arr] (x) -- (W1) -- (phi) -- (W2) -- (y);

\node[mat, below=8mm of W1] (PW1) {$P W_1$};
\node[mat, right=10mm of PW1] (phi2) {$\phi$};
\node[mat, right=10mm of phi2] (W2P) {$W_2 P^{-1}$};

\draw[arr] ($(x.south)+(0,-8mm)$) -- (PW1) -- (phi2) -- (W2P) -- ($(y.south)+(0,-8mm)$);

\node at ($(phi)!0.5!(W2)$) [above=4mm] {$\;f(x)=W_2\phi(W_1x)$};
\node at ($(phi2)!0.5!(W2P)$) [below=4mm] {$\;f'(x)=(W_2P^{-1})\phi(PW_1x)=f(x)$};

\end{tikzpicture}
\caption{Permuting hidden units and compensating in the next layer preserves function.}
\label{fig:mlp-perm}
\end{figure}
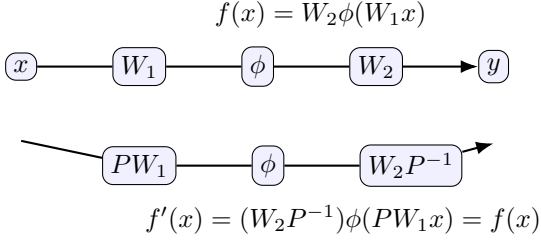

\paragraph{Multi-head mixing.}
Block-invertible head mixing can be absorbed by $W_O$ (approximate invariance). We use compensated transforms as a diagnostic; visual results in \S\ref{sec:diagnostics}.

\subsection{Task-level invariances and the groupoid view}
\label{sec:method-taskinv}
\paragraph{Groupoid framing.}
Semantics-preserving transforms (e.g., $\alpha$-renaming, algebraic rewrites) compose partially and may be non-invertible; we evaluate invariance over \emph{orbits}.

\begin{figure}[th]
\centering
\begin{subcaptionblock}{0.45\linewidth}
\begin{lstlisting}[language=Python,basicstyle=\ttfamily\small]
def add(a, b):
    return a + b
x = add(3, 5)
\end{lstlisting}
\caption{Original}
\end{subcaptionblock}\hfill
\begin{subcaptionblock}{0.45\linewidth}
\begin{lstlisting}[language=Python,basicstyle=\ttfamily\small]
def add(u, v):
    return u + v
y = add(3, 5)
\end{lstlisting}
\caption{Renamed}
\end{subcaptionblock}
\caption{Alpha-renaming orbit for code: predictions (e.g., pass@k) should be invariant.}
\label{fig:alpha-rename}
\end{figure}

\paragraph{Orbit construction and D5.}
Construct orbits and measure the \emph{invariance ratio} (IR) per input; log a failure taxonomy (predictive use in \S\ref{sec:diagnostics}).

\subsection{Commutators and inverse-free curvature (definitions only)}
\label{sec:method-comm-fiber}
\paragraph{Activation commutator (order sensitivity).}
For submodules $A,B$ acting on a calibration batch $X$,
\[
\Delta_{A,B}(X)=\|A(B(X))-B(A(X))\|_F,
\]
used as a map of order-sensitive pairs (see \S\ref{sec:diagnostics}).

\paragraph{Inverse-free curvature (holonomy surrogate).}
To avoid ill-posed inverses from LN/softmax/MLP, we compare two JVP paths that end in the same fiber:
\begin{equation}
\label{eq:invfree-kappa}
\kappa_{\mathrm{inv}}(i,j,\ell)^2
:= \mathbb{E}_{v}\,\big\|\,T^{\mathrm{layer}}_{i,\ell}T^{\mathrm{attn}}_{i\leftarrow j,\ell}v
- T^{\mathrm{attn}}_{i\leftarrow j,\ell+1}T^{\mathrm{layer}}_{j,\ell}v\,\big\|_2^2.
\end{equation}
The Frobenius norm used in Eq.~\ref{eq:invfree-kappa} is preserved under orthogonal 
coordinate changes. Computation, sparsity, and visualizations

\noindent\textbf{Defaults.} Unless stated otherwise, we use the settings in Table~\ref{tab:exp-setup}: probes $r{=}6$, targets/layer $k{=}8$, neighbors $m{=}6$, commutator threshold $\tau_\Delta{=}0.10$, curvature threshold $\tau_\kappa{=}0.12$, and orbit tolerance $\delta{=}0.02$.

\subsection{Cost, sparsity, and sampling}
\label{sec:method-cost}
Let $d$ be width, $L$ layers, $T$ tokens, $r$ Hutchinson probes, $k$ sampled targets/layer, and $m$ neighbors/target. Matrix-free JVPs yield
\[
\text{work}\;\approx\;\mathcal{O}\!\big(r\cdot L\cdot k\cdot m\cdot \mathrm{cost}_{\mathrm{JVP}}\big).
\]
We (i) sample $k$ tokens/layer (uniform or by saliency), (ii) keep top-$m$ neighbors by attention mass, (iii) \emph{scan} with frozen-softmax transports and \emph{confirm} hotspots with full JVPs, and (iv) batch $r$ probes via vmap. Defaults: $r{=}6$, $k{=}8$, $m{=}6$.
\section{Geometry-based Diagnostics}
\label{sec:diagnostics}

This section operationalizes the signals introduced in \S\ref{sec:method}: (i) \emph{commutators} that surface order-sensitive module pairs and (ii) \emph{inverse-free curvature} that localizes non-commuting transports on the $(\text{position},\text{layer})$ grid. Unless stated, we use the defaults in Table~\ref{tab:exp-setup} (probes $r{=}6$, targets/layer $k{=}8$, neighbors $m{=}6$, thresholds $\tau_\Delta{=}0.10$, $\tau_\kappa{=}0.12$, orbit tolerance $\delta{=}0.02$).

\subsection{Commutators as order-sensitivity diagnostics}
\label{sec:commutators}
\paragraph{Definition and intent.}
For submodules $A,B$ acting on a calibration batch $X$, define
\[
\Delta_{A,B} \;=\; \big\|\,A\!\circ\!B(X)\;-\;B\!\circ\!A(X)\,\big\|_F,
\]
a model-agnostic indicator of \emph{order sensitivity}. Large $\Delta_{A,B}$ suggests reordering/fusing $A,B$ risks output drift, while small $\Delta_{A,B}$ flags candidates for safe fusion or parallel execution.

\paragraph{Computation recipe.}
Choose a module granularity (e.g., attention heads within a layer, or whole sublayers). Evaluate $A(B(X))$ and $B(A(X))$ on the same $X$ (match seeds and dropout state if applicable) and aggregate per pair $(A,B)$ with $\|\cdot\|_F$. Populate a commutator matrix and visualize as a heatmap (Fig.~\ref{fig:comm-heatmap}).

\paragraph{From signal to action.}
Sort pairs by $\Delta_{A,B}$; above a guard $\tau_\Delta$ run sequentially (and log extra verifiers), below $\tau_\Delta$ allow fusion/reorder/parallelization. For intervention trials, define \emph{output drift} $\delta=\|y_{AB}-y_{BA}\|$ under a safe reorder/fuse and report $\rho(\Delta,\delta)$ (trend illustrated in \figref{fig:diagnostic-suite}(b), see \S\ref{sec:metrics}).

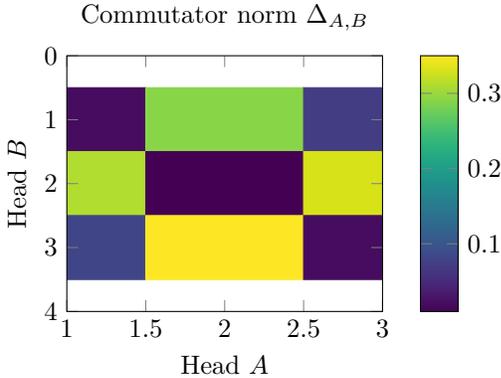
\begin{figure}[th]
\centering
\begin{tikzpicture}
\begin{axis}[
  width=0.7\linewidth,
  view={0}{90},
  colorbar,
  xlabel=Head $A$,
  ylabel=Head $B$,
  title={Commutator norm $\Delta_{A,B}$},
  colormap/viridis,
  enlargelimits=false,
]
\addplot[matrix plot, mesh/rows=3, mesh/cols=3, point meta=explicit] table[meta=z] {
x y z
1 1 0.02
1 2 0.31
1 3 0.08
2 1 0.29
2 2 0.01
2 3 0.35
3 1 0.07
3 2 0.33
3 3 0.02
};
\end{axis}
\end{tikzpicture}
\caption{Heatmap of $\Delta_{A,B}=\|A(B(X))-B(A(X))\|_F$ across attention heads (example data).}
\label{fig:comm-heatmap}
\vspace{-.1in}
\end{figure}

\paragraph{Artifacts.}
Emit \texttt{commutator.csv} with columns \texttt{module\_A, module\_B, delta\_fro}, and a static heatmap for CI. Downstream systems consume $\Delta_{A,B}$ and the guard $\tau_\Delta$.

\subsection{Fiber-bundle perspective: discrete transports and holonomy}
\label{sec:fiber}
\paragraph{Set-up.}
Let the base be $B=\text{Positions}\times\text{Layers}$; each node $(i,\ell)$ carries a fiber $\mathbb{R}^d$. Define \emph{vertical} transports $T^{\text{layer}}_{i,\ell}\approx \partial h_{i,\ell+1}/\partial h_{i,\ell}$ and \emph{horizontal} transports $T^{\text{attn}}_{i\leftarrow j,\ell}\approx \partial h^{\text{out}}_{i,\ell}/\partial h^{\text{in}}_{j,\ell}$ (edge-wise; sparse neighbors). Curvature summarizes the non-commutativity of these transports around small loops; \emph{flat} regions are reordering-safe while \emph{hotspots} are risk zones.

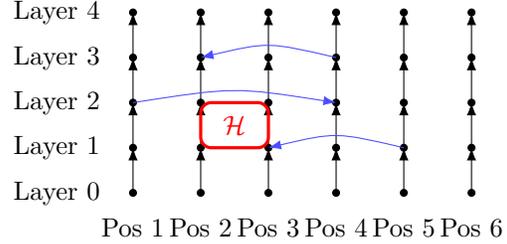
\begin{figure}[t]
\centering
\begin{tikzpicture}[>=Latex, x=\DX cm, y=\DY cm]
  \foreach \i in {1,...,\Npos} {
    \foreach \l in {0,...,\numexpr\Nlayers-1\relax} {
      \filldraw[black] (\i,\l) circle (\DotSize);
    }
  }
  \foreach \l in {0,...,\numexpr\Nlayers-1\relax} {
    \node[anchor=east] at (0.65,\l) {Layer \l};
  }
  \foreach \i in {1,...,\Npos} {
    \node[anchor=north] at (\i,-0.35) {Pos \i};
  }
  \foreach \i in {1,...,\Npos} {
    \foreach \l in {0,...,\numexpr\Nlayers-2\relax} {
      \draw[->] (\i,\l) -- (\i,\l+1);
    }
  }
  \newcommand{\AttnArrow}[3]{%
    \pgfmathsetmacro{\j}{#1} \pgfmathsetmacro{\ell}{#2} \pgfmathsetmacro{\ii}{#3}
    \path let \p1=(\j,\ell), \p2=(\ii,\ell) in
      coordinate (M) at ($(\p1)!.5!(\p2)$);
    \draw[->,blue!70]
      (\j,\ell) .. controls ($(M)+(0,\AttnLift)$) and ($(M)+(0,\AttnLift)$) .. (\ii,\ell);
  }
  \AttnArrow{1}{2}{4}
  \AttnArrow{4}{3}{2}
  \AttnArrow{5}{1}{3}
  \draw[red,very thick,rounded corners]
    (\LoopIStart,\LoopLStart) --
    (\LoopIEnd,\LoopLStart) --
    (\LoopIEnd,\LoopLEnd) --
    (\LoopIStart,\LoopLEnd) -- cycle;
  \node[red] at ({(\LoopIStart+\LoopIEnd)/2},{(\LoopLStart+\LoopLEnd)/2}) {$\mathcal{H}$};
\end{tikzpicture}
\caption{Product base \(B=\text{Positions}\times\text{Layers}\). Vertical edges are layer transports \(T^{\mathrm{layer}}_{i,\ell}\); blue curves indicate sample attention transports \(T^{\mathrm{attn}}_{i\leftarrow j,\ell}\). The red rectangle marks a Wilson loop \(\mathcal{H}\) used to measure curvature (order sensitivity).}
\label{fig:fiber-grid}
\vspace{-.1in}
\end{figure}

\paragraph{Curvature maps.}
Aggregate per-loop curvature (below) into position/layer maps (\figref{fig:holonomy-scatter}) and summarize predictiveness by ROC (\figref{fig:diagnostic-suite}(a)); see \S\ref{sec:metrics}.

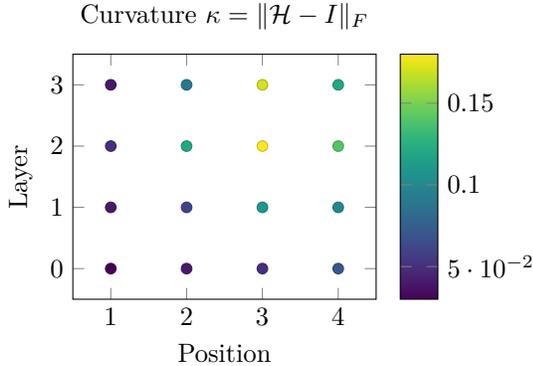
\begin{figure}[th]
\centering
\begin{tikzpicture}
\begin{axis}[
  width=0.68\linewidth,
  colorbar,
  xlabel=Position,
  ylabel=Layer,
  title={Curvature $\kappa=\|\mathcal{H}-I\|_F$},
  colormap/viridis,
  xmin=0.5, xmax=4.5, ymin=-0.5, ymax=3.5,
]
\addplot[
  scatter,
  only marks,
  scatter src=explicit,
  mark=*,
] table[
  row sep=\\,
  x=x, y=y, meta=z,
] {
x y z\\
1 0 0.03\\
2 0 0.04\\
3 0 0.05\\
4 0 0.07\\
1 1 0.04\\
2 1 0.06\\
3 1 0.11\\
4 1 0.10\\
1 2 0.05\\
2 2 0.12\\
3 2 0.18\\
4 2 0.14\\
1 3 0.04\\
2 3 0.09\\
3 3 0.17\\
4 3 0.12\\
};
\end{axis}
\end{tikzpicture}
\caption{Pointwise curvature over (position, layer); color encodes $\kappa$.}
\label{fig:holonomy-scatter}
\vspace{-.1in}
\end{figure}

\subsubsection{Discrete transports: implementation and what curvature measures}
\label{sec:transports-impl}
\paragraph{Vertical maps.}
In pre-LN blocks,
$T^{\text{layer}}_{i,\ell}\!\approx\! I + J^{\text{Attn}}_{i,\ell} + J^{\text{MLP}}_{i,\ell}$,
so residual connections make vertical transport close to identity; curvature therefore captures the \emph{accumulated deviation from identity} around a loop.

\paragraph{Horizontal maps.}
For each candidate edge $(j{\to}i,\ell)$, estimate $T^{\text{attn}}_{i\leftarrow j,\ell}$ by (a) a fast \emph{frozen-softmax} scan using per-head attention weights and $W_V,W_O$, and (b) confirm hotspots with full JVPs. We keep the top-$m$ neighbors by attention mass (with random exploration).

\subsubsection{Inverse-free curvature}
\label{sec:invfree}
\paragraph{Why inverse-free.}
LayerNorm, softmax, and MLPs are not globally invertible, so we avoid $(\cdot)^{-1}$ by comparing two paths that end in the same fiber $(i,\ell{+}1)$:
\[
\kappa_{\mathrm{inv}}(i,j,\ell)^2
= \mathbb{E}_{v}\,\big\|\,T^{\mathrm{layer}}_{i,\ell}T^{\mathrm{attn}}_{i\leftarrow j,\ell}v
- T^{\mathrm{attn}}_{i\leftarrow j,\ell+1}T^{\mathrm{layer}}_{j,\ell}v\,\big\|_2^2,
\]
as defined in Eq.~\eqref{eq:invfree-kappa}. We estimate the expectation with $r$ Rademacher probes via JVPs.

\paragraph{Loop selection and aggregation.}
Sample $k$ target tokens per layer and top-$m$ neighbors to form loops. Aggregate per input with either the maximum $\kappa_{\max}$ or the 95th percentile $\operatorname{p95}(\{\kappa\})$ to obtain a scalar failure predictor; see \S\ref{sec:metrics} and Fig.~\ref{fig:diagnostic-suite}(a).


\subsubsection{Cost, sparsity, and sampling}
\label{sec:cost-sampling}
\paragraph{Work model.}
With width $d$, layers $L$, tokens $T$, probes $r$, targets/layer $k$, and neighbors $m$,
\[
\text{work}\;\approx\;\mathcal{O}\!\big(r\,L\,k\,m\cdot \mathrm{cost}_{\mathrm{JVP}}\big).
\]
We (i) sample $k$ tokens/layer (uniform or saliency-biased), (ii) keep top-$m$ neighbors, (iii) scan with frozen transports and (iv) confirm hotspots with four JVPs/loop. Defaults: $r{=}6$, $k{=}8$, $m{=}6$.

\subsubsection{Notes and baselines}
\label{sec:notes}
\paragraph{Locality.}
“Locally flat’’ refers to $1{\times}1$ rectangles; $2{\times}2$ loops appear in ablations (rankings typically stable, magnitudes amplified).

\paragraph{Vertical composition.}
Our $T^{\text{layer}}$ already includes residual and sublayers; a two-edge vertical ablation (residual-only then sublayer-only) produces similar hotspot rankings.

\paragraph{Nulls.}
We report curvature for random-initialized models and for shuffled-attention baselines to calibrate magnitudes and false-alarm rates.

\subsection{Gauge fixing for interpretability and CI}
\label{sec:gaugefix}
\paragraph{When and why.}
We apply whitening and Procrustes \emph{for logging and cross-seed comparability}, not when computing $\kappa_{\mathrm{inv}}$. Metrics include probe-accuracy variance across seeds and saliency-rank stability (Kendall-$\tau$). The pipeline (\figref{fig:gauge-pipeline}) standardizes features so CI dashboards are reproducible.

\subsection{Representation tendencies for linguistic transforms}
\label{sec:ling}
\paragraph{Offsets, not theorems.}
Across languages, morphological/syntactic relations often appear as approximately linear offsets in embedding subspaces; this is an empirical tendency, not a claim of an explicit grammar-group representation. We evaluate offset consistency under orthogonal alignment and compare locations of offset-sensitive examples with commutator/curvature hotspots (\figref{fig:diagnostic-suite}(b)).

\medskip
\noindent\textbf{Artifacts and usage.}
This section yields: \texttt{commutator.csv} (pairwise $\Delta_{A,B}$), \texttt{holonomy.csv} (per-loop $\kappa_{\mathrm{inv}}$ with indices $(i,j,\ell)$), and static plots for CI (\figref{fig:comm-heatmap}, \figref{fig:fiber-grid}, \figref{fig:holonomy-scatter},~\figref{fig:diagnostic-suite}(a),~\figref{fig:diagnostic-suite}(b)). Systems consume these along with guard thresholds $(\tau_\Delta,\tau_\kappa)$ to decide when to fuse, reorder, or parallelize, and when to enforce extra verification.
\section{Evaluation Metrics}
\label{sec:metrics}

\newcommand{\subcapdrop}{8mm}
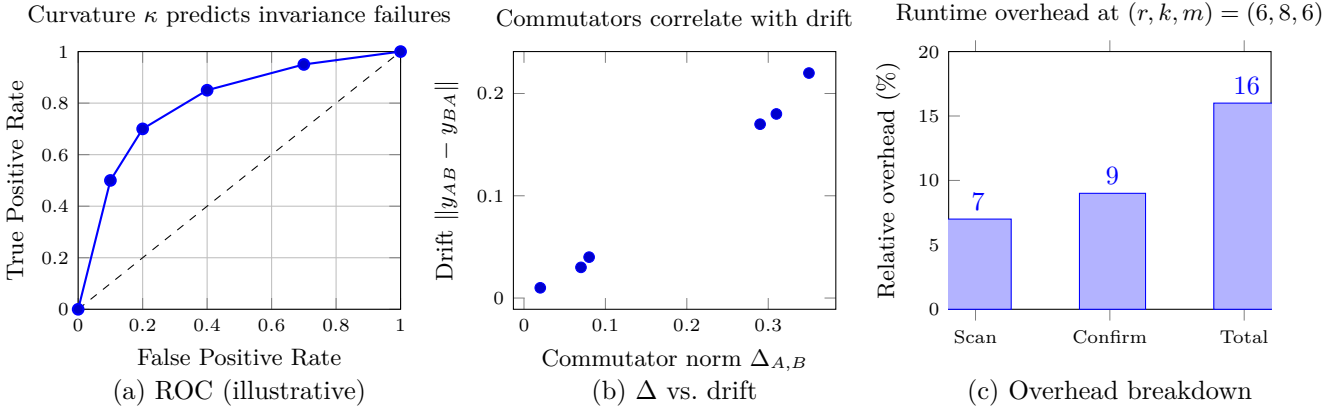
\begin{figure*}[t]
\centering
\begin{tikzpicture}
\begin{groupplot}[
  group style={group size=3 by 1, horizontal sep=15mm},
  scale only axis=true,
  width=0.25\textwidth,
  height=0.20\textwidth,
  title style={font=\footnotesize},
  label style={font=\footnotesize},
  tick label style={font=\scriptsize}
]

\nextgroupplot[
  xlabel=False Positive Rate,
  ylabel=True Positive Rate,
  title={Curvature $\kappa$ predicts invariance failures},
  xmin=0, xmax=1, ymin=0, ymax=1, grid=both
]
\addplot+[thick] coordinates {(0,0) (0.1,0.5) (0.2,0.7) (0.4,0.85) (0.7,0.95) (1,1)};
\addplot[dashed] coordinates {(0,0) (1,1)};
\label{fig:roc-curvature}

\nextgroupplot[
  xlabel={Commutator norm $\Delta_{A,B}$},
  ylabel={Drift $\|y_{AB}-y_{BA}\|$},
  title={Commutators correlate with drift},
]
\addplot+[only marks] coordinates {
(0.02,0.01) (0.08,0.04) (0.31,0.18) (0.35,0.22) (0.07,0.03) (0.29,0.17)
};
\label{fig:comm-drift}

\nextgroupplot[
  ybar,
  ylabel={Relative overhead (\%)},
  symbolic x coords={Scan,Confirm,Total},
  xtick=data,
  nodes near coords,
  ymin=0, ymax=20,
  title={Runtime overhead at $(r,k,m)=(6,8,6)$},
  bar width=25pt, bar shift=1pt
]
\addplot coordinates {(Scan,7) (Confirm,9) (Total,16)};
\label{fig:overhead-bars}

\end{groupplot}

\node[anchor=north] at ($(group c1r1.south)+(0,-\subcapdrop)$) {(a) ROC (illustrative)};
\node[anchor=north] at ($(group c2r1.south)+(0,-\subcapdrop)$) {(b) $\Delta$ vs.\ drift};
\node[anchor=north] at ($(group c3r1.south)+(0,-\subcapdrop)$) {(c) Overhead breakdown};
\end{tikzpicture}
\vspace{-.3in}
\caption{Diagnostic suite: (a) curvature ROC; (b) commutator–drift correlation; (c) scan/confirm/total overhead. Replace illustrative values with measured results.}
\label{fig:diagnostic-suite}
\vspace{-.1in}
\end{figure*}

We score the proposed diagnostics along six axes: strict invariance under clean transforms (IR), curvature–failure predictiveness (ROC/AP), output sensitivity to operator reordering (drift), reproducibility under gauge fixing, RoPE phase-drift stability, and runtime overhead. Each metric is defined at the \emph{per-input} level and as \emph{dataset aggregates}. Diagnostic/topology maps appear in \S\ref{sec:diagnostics}; this section defines how we turn those signals into reportable numbers. Unless stated otherwise, we use the defaults in Table~\ref{tab:exp-setup}.

\paragraph{Invariance ratio (E1; ties to Fig.~\ref{fig:alpha-rename}).}
For an input $x$ with an orbit $\mathcal{O}(x)$ (e.g., $\alpha$-renamings; Fig.~\ref{fig:alpha-rename}), define
\begin{footnotesize}
\[
\mathrm{IR}(x)=\frac{1}{|\mathcal{O}(x)|}\sum_{x'\in\mathcal{O}(x)} \mathbf{1}\{f(x')=f(x)\},
\]
\end{footnotesize}
where $f(\cdot)$ is a task decision (e.g., pass@k for code, argmax token for classification). Assign a failure label $y{=}1$ if $\mathrm{IR}(x) < 1 - \delta$, where $\delta$ is a tolerance calibrated on a benign validation set.  
\text{Aggregates:} macro-average
\begin{footnotesize}
\[
\overline{\mathrm{IR}}_{\text{macro}}=\frac{1}{N}\sum_{x} \mathrm{IR}(x),
\]
\[
\overline{\mathrm{IR}}_{\text{micro}}=\frac{\sum_{x}\sum_{x'\in\mathcal{O}(x)} \mathbf{1}\{f(x')=f(x)\}}{\sum_{x} |\mathcal{O}(x)|},
\]
\end{footnotesize}
with stratified bootstrap 95\% CIs.

\paragraph{Curvature summary (E5; ties to Figs.~\ref{fig:fiber-grid}, \ref{fig:holonomy-scatter}).}
Given loops $\mathcal{L}(x)$ associated with $x$, aggregate inverse-free curvature (\S\ref{sec:invfree}) as
\[
\kappa_{\max}(x)=\max_{(i,j,\ell)\in\mathcal{L}(x)} \kappa_{\mathrm{inv}}(i,j,\ell),
\]
\[
\kappa_{p95}(x)=\mathrm{quantile}_{0.95}\!\left(\{\kappa_{\mathrm{inv}}(i,j,\ell)\}_{(i,j,\ell)\in\mathcal{L}(x)}\right).
\]
We visualize position/layer curvature in \S\ref{sec:diagnostics} (Fig.~\ref{fig:holonomy-scatter}) and use these aggregates as predictors below.

\paragraph{Predictive power for invariance breaks (E6).}
Treat a curvature aggregate $s(x)\!\in\!\{\kappa_{\max}(x),\kappa_{p95}(x)\}$ as a score for predicting $y$. Sweep thresholds to obtain ROC and PR curves; report ROC AUC, Average Precision (AP), Brier score, and Expected Calibration Error (ECE; 10 bins). Use stratified bootstrap for 95\% CIs. Include null baselines: (i) input-wise random permutation of curvature maps; (ii) random-initialized networks matched by width/depth.

\paragraph{Order sensitivity and drift (E3; ties to Figs.~\ref{fig:comm-heatmap}, \ref{fig:diagnostic-suite}(b)).}
Compute the commutator norm $\Delta_{A,B}=\|A\!\circ\!B - B\!\circ\!A\|_F$ on a calibration batch and measure \emph{output drift} under safe reorder/fuse, $\delta(x)=\|y_{AB}-y_{BA}\|$.  
\textbf{Aggregates:} report Spearman $\rho(\Delta,\delta)$, Pearson $r$, and a robust slope via Theil–Sen, all with bootstrap CIs. When relevant, report partial correlations controlling for sequence length or output entropy.


\paragraph{Gauge-stable logging (E2/E7; ties to Fig.~\ref{fig:gauge-pipeline}).}
Across random seeds and identical data slices, compare: (i) probe-accuracy variance (lower is better), (ii) saliency-rank stability via Kendall-$\tau$ (higher is better), and (iii) cross-seed cosine distance of layerwise means (lower is better). Report pre/post gauge-fix (whiten\,$+$\,Procrustes) deltas with CIs. In CI, set acceptance thresholds (e.g., variance ratio $\le 0.6$) that fail builds when reproducibility regresses.

\paragraph{RoPE phase-drift stability (E4; ties to Fig.~\ref{fig:rope-rotation}).}
Apply controlled phase offsets $\delta\theta$ to RoPE frequencies; measure Wasserstein (or symmetric KL) between attention-score distributions at depth $\ell$ \emph{with} vs.\ \emph{without} the offset. Summarize by area-under-drift vs.\ depth/context length; larger areas indicate long-context brittleness.

\paragraph{Runtime overhead (all experiments).}
Measure end-to-end overhead of diagnostics at defaults $(r,k,m)$ and at reduced settings. Attribute cost to scan vs.\ confirm (frozen-softmax multiplies vs.\ JVPs). Target $\le 10$–$20\%$ overhead at defaults; additionally report throughput gains unlocked by curvature/commutator-informed fusion (\S\ref{sec:diagnostics}).

\vspace{1ex}
\noindent\textbf{Reporting checklist.} Always include: $N$ inputs, loops/input, seeds, model sizes; tolerances $\delta$; bootstrap details (stratification, reps); nulls; confidence intervals; ablation knobs (loop size, gauge-fix mode, transport mode); and wall-clock environment (GPU/TPU type, batch size).

\begin{figure*}[th!]
\centering
\begin{tikzpicture}[node distance=10mm, >=Latex]
\tikzstyle{blk}=[draw, rounded corners, fill=gray!27,
                 minimum height=10mm, text width=34mm, align=center]
\tikzstyle{sig}=[draw, rounded corners, fill=yellow!30,
                 minimum height=10mm, text width=40mm, align=center]

\node[blk] (groupkit) {LLM-GroupKit \\ (E1--E7 diagnostics)};
\node[sig, right=8mm of groupkit] (signals)
  {Signals (CSV):\\ commutator $\Delta$, curvature $\kappa$, invariance ratio, gauge stats};
\node[blk, right=36mm of signals, yshift=11mm] (tb) {LLM TestBench \\ (generator, oracle, fuzzer)};
\node[blk, right=36mm of signals, yshift=-11mm] (ci) {CI / Release Gate};

\draw[->] (groupkit) -- (signals);
\draw[->] (signals) -- (tb)
  node[midway, above, align=center, text width=32mm]
  {prioritize orbits,\\choose canaries};
\draw[->] (signals) -- (ci)
  node[midway, below, align=center, text width=36mm]
  {thresholds,\\regression checks};
\draw[->] (tb) -- (ci)
  node[midway, right, align=center, text width=26mm]
  {test outcomes,\\counterexamples};

\end{tikzpicture}
\caption{Diagnostics feed an LLM-driven TestBench and CI gate.}
\label{fig:systems-wiring}
\vspace{-.1in}
\end{figure*}
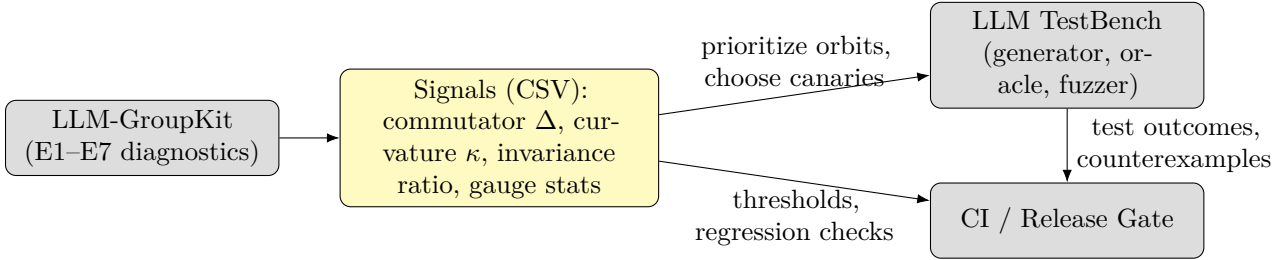

\vspace{-.1in}
\section{Systems Implications}
\label{sec:systems}

Our diagnostics expose \emph{geometry-derived signals} that an \textbf{LLM-driven TestBench} can consume to generate, prioritize, and evaluate tests. Figure~\ref{fig:systems-wiring} shows the wiring: \textbf{LLM-GroupKit} (E1--E7) produces signals; a \textbf{TestBench} uses them to steer metamorphic testing (e.g., $\alpha$-renaming orbits), reordering canaries, and fuzzing; a \textbf{CI/Release Gate} enforces thresholds.

\vspace{-.06in}
\subsection{Signal interfaces and data contracts}
\label{sec:systems-interfaces}
Diagnostics emit append-only, schema-versioned CSVs designed to be \emph{gauge-stable}. Features are logged \emph{after} gauge-fixing (whiten+Procrustes), except raw $\kappa_{\mathrm{inv}}$ which uses orthogonal projections.
\begin{itemize}[leftmargin=1.2em,itemsep=-0.5em,topsep=-0.5em]
\item \textbf{holonomy.csv}: \texttt{position:int, layer:int, kappa:float, model:str, seed:int, ts:iso8601, schema:int}
\item \textbf{commutator.csv}: \texttt{i:int, j:int, value:float, block:str, model:str, seed:int, ts:iso8601, schema:int}
\item \textbf{ir.csv}: \texttt{input\_id:str, IR:float, tol:float, label:int, model:str, seed:int, ts:iso8601, schema:int}
\item \textbf{gauge\_stats.csv}: \texttt{layer:int, seed:int, kendall\_tau:float, probe\_var:float, model:str, ts:iso8601, schema:int}
\end{itemize}
Artifacts include model hash, seeds, and timestamps for forensic replay.

\vspace{-.06in}
\subsection{From signals to actions: TestBench policies}
\label{sec:systems-policy}
We map signals to conservative defaults (Table~\ref{tab:signal-actions-tb}); the TestBench can override per workload.

\begin{table*}[ht]
\centering
\caption{Signal-to-action mapping for an LLM-driven TestBench and CI.}
\label{tab:signal-actions-tb}
\small
\begin{tabular}{@{}lll@{}}
\toprule
Signal & Consumer & Default action \\
\midrule
Low $\Delta_{A,B}$ & TestBench & Generate/favor fuse+reorder canaries; expect low drift \\
High $\Delta_{A,B}$ & TestBench & Mark reorder-unsafe; synthesize AB/BA counterexamples \\
Low $\kappa_{\max/p95}$ & TestBench & Run orbits in parallel; increase fuzz budget \\
High $\kappa_{\max/p95}$ & TestBench & Sequentialize orbits; add stronger oracles/checkers \\
IR failures & TestBench & Escalate severity; auto-generate minimal counterexample \\
Gauge-stable $\downarrow$ variance & CI/Gate & Accept build; freeze probe config \\
\bottomrule
\end{tabular}
\end{table*}

\paragraph{Example: curvature-aware test generation.}
\begin{lstlisting}[
  language=Python,
  basicstyle=\ttfamily\footnotesize,
  breaklines=true,
  breakatwhitespace=true,
  columns=fullflexible,
  keepspaces=true,
  postbreak=\mbox{\textcolor{gray}{$\hookrightarrow$}\space},
  caption={Curvature-guided test planning with an LLM generator and oracle.},
]
def plan_tests(inputs, kappa_map, delta_map, tau_kappa=0.12, tau_delta=0.10):
    tests = []
    for x in inputs:
        # prioritize positions/layers with high curvature
        hot = select_hotspots(kappa_map[x], quantile=0.95, thr=tau_kappa)
        # metamorphic orbits: alpha-rename, algebraic rewrites, paraphrases
        orbits = llm_generate_orbits(x, hotspots=hot, budget="adaptive")
        # decide whether to include reorder/fuse canaries
        if mean_delta(delta_map[x]) <= tau_delta:
            tests.append(make_fuse_reorder_canaries(x))
        tests.extend(orbits)
    return tests
\end{lstlisting}

\vspace{-.1in}
\subsection{Deployment patterns}
\label{sec:systems-runtime}
Pick a pattern by latency budget:
\begin{enumerate}[leftmargin=1.2em,itemsep=-0.5em,topsep=-0.5em]
\item \textbf{Sidecar (recommended):} a co-process computes signals and enqueues prioritized tests; the TestBench drains with a 200--500\,ms window.
\item \textbf{In-service (low latency):} inline light scans (\emph{frozen-softmax}); rate-limit heavy JVP confirmation.
\item \textbf{Batch (offline):} nightly full maps for dashboards; thresholds updated via config.
\end{enumerate}
Use Fig.~\ref{fig:diagnostic-suite}(c) (\S\ref{sec:metrics}) to budget; keep \texttt{Total} $\le 20\%$.

\vspace{-.06in}
\subsection{Observability and CI gates}
\label{sec:systems-ci}
Dashboards track ROC AUC/AP (\S\ref{sec:metrics}), $\rho(\Delta,\delta)$, variance/rank stability, and overhead. CI fails on regressions (e.g., variance ratio $>0.6$ post gauge-fix, AUC drop $>0.03$). Each build snapshots model hash, seeds, knob settings, and schema versions.

\vspace{-.06in}
\subsection{Failure modes and mitigations}
\label{sec:systems-failure}
\textbf{Stale signals:} expire maps after $T$ minutes or model-hash change.  
\textbf{Threshold drift:} recalibrate (Platt/Isotonic) on fresh labels.  
\textbf{Nondeterminism:} fix seeds; log RNG states; align via Procrustes before comparisons.  
\textbf{Long-range effects:} enlarge loop radius if IR failures appear off-hotspot.

\vspace{-.1in}
\subsection{Knobs and defaults}
\label{sec:systems-knobs}
Expose: probes $r{=}6$, targets/layer $k{=}8$, neighbors $m{=}6$, $\tau_{\Delta}{=}0.10$, $\tau_{\kappa}{=}0.12$, tolerance $\delta{=}0.02$ (see also Table~\ref{tab:exp-setup}); sweep in ablations.

\vspace{-.06in}
\subsection{End-to-end flows}
\label{sec:systems-flows}
\textbf{Test-first release:} compute signals $\rightarrow$ generate curvature/commutator-guided tests $\rightarrow$ run orbits+canaries $\rightarrow$ gate on IR/ROC/overhead thresholds.  
\textbf{Reactive repair:} on failure, auto-synthesize minimal counterexample, suggest patch, and re-run targeted orbits; persist gauge-fixed traces for audit.

\paragraph{Takeaway.} Curvature and commutator signals give an \emph{LLM-first testing substrate}: they prioritize which tests to run, where to search for counterexamples, and when a build is safe to ship.

\vspace{-.06in}
\subsection{High-confidence applications for LLM systems}
\label{sec:llm-applications}

Beyond general deployment hygiene, our diagnostics enable specific reliability improvements in LLM production systems.

\paragraph{Prompt robustness testing.}
Measure $\kappa_{\mathrm{inv}}$ under semantics-preserving prompt variations such as paraphrases, template changes, or role reorderings. High curvature flags brittle prompts that need stabilization or explicit invariance constraints. \emph{Example:} customer-service prompts should maintain low $\kappa_{\mathrm{inv}}$ across rephrasings like “Please help me” versus “I need assistance.” Curvature spikes indicate training gaps.

\paragraph{RAG passage ordering sensitivity.}
When retrieval returns ranked documents \cite{lewis2020retrieval}, passage order can spuriously affect outputs due to position bias. Commutators $\Delta_{A,B}$ between passage positions quantify reordering risk. Systems can: (i) shuffle low-$\Delta$ passages freely for parallelism, (ii) preserve high-$\Delta$ ordering with sequential processing, and (iii) add order-invariance verification for critical queries.

\paragraph{Fine-tuning stability audits.}
Task-specific fine-tuning can degrade invariances. WILSON provides before and after curvature maps. Regions where $\kappa_{\mathrm{inv}}$ increases post-tuning flag potential regressions. CI gates reject runs that increase curvature beyond accepted thresholds on held-out metamorphic test suites.

\paragraph{Multi-turn conversation drift.}
Long conversations accumulate context \cite{PathAGIChang2024,SocraSynthChangCSCI2023,chang2025MACIDD}. Curvature maps identify layers and positions where context integration becomes order-sensitive. Systems can trigger summarization when $\kappa_{\mathrm{inv}}$ rises above a threshold, route high-$\kappa$ turns to sequential processing, and log drift for post-hoc debugging.

\paragraph{Chain-of-thought pathway stability.}
Different CoT \cite{WeiCoT2022} reasoning traces (e.g., ``step by step'' versus ``let us think'') should yield consistent answers. Commutators between chain-of-thought templates measure pathway sensitivity. High $\Delta$ indicates fragile reasoning that benefits from ensemble or self-consistency decoding.

\paragraph{Confidence estimates.}
Curvature serves as an uncertainty proxy. Predictions from high-$\kappa$ regions have higher variance under benign perturbations. Systems can surface low-confidence warnings to users, route uncertain queries to human review, and trigger additional verification for high-$\kappa$ outputs.

\paragraph{Integration pattern.}
Applications follow a common workflow:
\begin{enumerate}[leftmargin=1.2em,itemsep=-0.5em,topsep=-0.5em]
\item Compute curvature and commutator maps.
\item Set domain-specific thresholds via calibration.
\item Emit signals (as metrics) to the orchestration layer.
\item Take action: reject, retry, verify, or route differently.
\end{enumerate}

All applications require only the diagnostic signals described earlier in Section~4, with no model retraining or architectural changes.
\section{Experiments}
\label{sec:experiments}

\subsection{Scope and preregistration}
\label{sec:scope-prereg}
This paper reports \emph{no measured results}. We preregister a complete, executable evaluation plan for the diagnostics introduced in \S\ref{sec:method} and visualized in \S\ref{sec:diagnostics}, scored by the metrics in \S\ref{sec:metrics}. Unless stated otherwise, we use the defaults in Table~\ref{tab:exp-setup} for all diagnostics. Code and schemas follow the CSV contracts in \S\ref{sec:systems-interfaces}.

\begin{table*}[t!]
\caption{Consolidated experimental setup and diagnostics knobs used across all planned experiments.
Thresholds are calibrated on a validation set and fixed for test reporting.}
\label{tab:diag-setup}
\label{tab:exp-setup}
\centering
\setlength{\tabcolsep}{6pt}
\footnotesize
\begin{tabular}{@{}lll@{}}
\toprule
Category & Item & Value / Notes \\
\midrule
\multicolumn{3}{@{}l}{\emph{General setup}}\\
& Models & Llama-2/3, Mistral, Gemma (7B–13B) \\
& Precision & BF16/FP16; FP32 for LN/softmax JVPs \\
& Context length & 2k–8k tokens \\
& Tokenizer & model default \\
& Decoding (code) & pass@k; $k\in\{1,10\}$; max gen length 256 tokens \\
& Seeds & 3–5; deterministic where available \\
\midrule
\multicolumn{3}{@{}l}{\emph{Diagnostics knobs}}\\
& Probes $r$ & 6 \;\;(range 4–12) \\
& Targets/layer $k$ & 8 \;\;(range 4–16) \\
& Neighbors $m$ & 6 \;\;(range 4–12) \\
& Comm. thresh $\tau_\Delta$ & 0.10 \;\;(0.05–0.20) \\
& Curv. thresh $\tau_\kappa$ & 0.12 \;\;(0.06–0.25) \\
& Orbit tol $\delta$ & 0.02 \;\;(0.00–0.05) \\
\bottomrule
\end{tabular}
\end{table*}

\subsection{Experimental setup}
\label{sec:setup}
\paragraph{Models.}
Open decoder-only Transformers with RoPE/relative positions: small/medium checkpoints (e.g., 7B--13B); inference in BF16/FP16 with FP32 pins for LayerNorm/softmax JVPs. We fix tokenizer/version and freeze weights throughout.

\paragraph{Tasks and datasets.}
(i) \textbf{Code orbits} for invariance: HumanEval/MBPP-style prompts with semantics-preserving $\alpha$-renaming and algebraic rewrites (Fig.~\ref{fig:alpha-rename}). (ii) \textbf{RoPE stress}: synthetic sequences for controlled phase offsets (Fig.~\ref{fig:rope-rotation}). (iii) \textbf{Calibration batches} for commutators/reorder probes (random and task-like inputs). All prompts are deduplicated; orbits are verified to preserve semantics.

\paragraph{Hardware.}
A100-class GPUs (40--80\,GB), PyTorch (\texttt{torch.func} JVPs), cudnn deterministic kernels where available. Batch sizes and context lengths are recorded in artifacts.

\paragraph{Seeds and determinism.}
We fix global seeds for Python/NumPy/PyTorch; log RNG states; reuse identical input slices across seeds for cross-seed comparability. Gauge-fix (whiten+Procrustes) is applied only for logging, not when computing $\kappa_{\mathrm{inv}}$.

\paragraph{Diagnostic knobs (defaults).}
Hutchinson probes $r{=}6$, targets per layer $k{=}8$, neighbors per target $m{=}6$, commutator threshold $\tau_{\Delta}{=}0.10$, curvature threshold $\tau_{\kappa}{=}0.12$, orbit tolerance $\delta{=}0.02$; see Table~\ref{tab:exp-setup}. 

\subsection{Planned procedures}
\label{sec:procedures}

\paragraph{E1: Invariance under clean transforms.}
For each input $x$, generate an orbit $\mathcal{O}(x)$ of semantics-preserving transforms (identifier renaming, algebraic equivalence). Compute the decision function $f(\cdot)$ (pass@k for code, argmax for classification) and the invariance ratio $\mathrm{IR}(x)$ defined in \S\ref{sec:metrics}. Emit \texttt{ir.csv} with labels $y{=}1$ if $\mathrm{IR}(x) < 1-\delta$. 

\paragraph{E2: Gauge-stable logging.}
On matched seeds/slices, compute probes/saliency both \emph{pre}- and \emph{post}-gauge-fix (Fig.~\ref{fig:gauge-pipeline}). Log probe-accuracy variance and Kendall-$\tau$ rank agreements to \texttt{gauge\_stats.csv}. 

\paragraph{E3: Commutator maps and reorder drift.}
For selected submodule pairs $(A,B)$, compute $\Delta_{A,B}=\|A\!\circ\!B-B\!\circ\!A\|_F$ on a calibration batch (Fig.~\ref{fig:comm-heatmap}). Apply \emph{safe} reorder/fuse interventions (no change in semantics, numerically stable) and measure output drift $\delta=\|y_{AB}-y_{BA}\|$; report $\rho(\Delta,\delta)$ and $r$ as in \S\ref{sec:metrics} and plot trends as in Fig.~\ref{fig:diagnostic-suite}(b).

\paragraph{E4: RoPE phase-drift.}
Apply controlled phase offsets $\delta\theta$ to RoPE; compute per-layer distribution distances between perturbed and unperturbed attention scores and summarize area-under-drift vs.\ depth/context length (definition in \S\ref{sec:metrics}).

\paragraph{E5: Curvature maps (holonomy).}
Compute inverse-free curvature $\kappa_{\mathrm{inv}}$ via Eq.~\eqref{eq:invfree-kappa} using JVPs: \emph{scan} with frozen-softmax transports to propose edges; \emph{confirm} hotspots with full JVP loops. Emit per-position/layer maps (Fig.~\ref{fig:holonomy-scatter}) to \texttt{holonomy.csv}.

\paragraph{E6: Curvature $\rightarrow$ failure prediction.}
Aggregate loop scores per input $x$ (e.g., $\kappa_{\max}$, $\kappa_{p95}$ from \S\ref{sec:metrics}); use these as scores for predicting $y$ from E1. Produce ROC/PR curves and report AUC/AP/Brier/ECE with bootstrap CIs (illustrative ROC in Fig.~\ref{fig:diagnostic-suite}a).

\paragraph{E7: Overhead and budget.}
Measure wall-clock of \emph{scan}, \emph{confirm}, and total overhead at defaults and reduced settings; attribute cost (matrix multiplies vs.\ JVPs) and report as in Fig.~\ref{fig:diagnostic-suite}c. Target $\le$\,20\% at defaults.

\subsection{Baselines and ablations}
\label{sec:ablations}
\paragraph{Null baselines.}
(i) \textbf{Permuted curvature maps:} randomly permute $\kappa_{\mathrm{inv}}$ across inputs before scoring E6; (ii) \textbf{Random-init network:} compute the full pipeline on width/depth-matched randomly initialized models.

\paragraph{Comparative baselines.}
Heuristics without geometry: (i) attention-entropy thresholds, (ii) gradient-norm canaries, (iii) layerwise activation variance. Compare AUC/AP and calibration vs.\ $\kappa$-based scores.

\paragraph{Ablation knobs.}
(i) Loop size: $1{\times}1$ vs.\ $2{\times}2$ rectangles; (ii) Transport mode: frozen-softmax scan only vs.\ scan+JVP confirm; (iii) Sampling: $(r,k,m)$ sweeps; (iv) Gauge-fix: none vs.\ whiten-only vs.\ whiten+Procrustes (for logging stability, not for $\kappa_{\mathrm{inv}}$).

\subsection{Metrics and statistical protocol}
\label{sec:exp-metrics}
We use the formal definitions in \S\ref{sec:metrics}: IR (per-input and macro/micro aggregates), ROC AUC \& AP (E6), $\rho(\Delta,\delta)$ and $r$ (E3), area-under-drift (E4), variance ratios and Kendall-$\tau$ (E2), and overhead breakdowns (E7). Confidence intervals use stratified bootstrap ($\ge$1000 resamples). Calibration uses 10-bin ECE with bin-wise Wilson intervals.

\subsection{Artifacts and reproducibility}
\label{sec:artifacts}
We export \texttt{holonomy.csv}, \texttt{commutator.csv}, \texttt{ir.csv}, and \texttt{gauge\_stats.csv} as append-only records (schemas in \S\ref{sec:systems-interfaces}). Scripts match the code in the \texttt{Documentation of Key Code} section; Colab notebooks reproduce all figures (\S\ref{sec:diagnostics}, \S\ref{sec:metrics}). Seeds, model hashes, knob settings, and schema versions are stored per run.

\subsection{Compute budget and runtime}
\label{sec:budget}
Complexity follows \S\ref{sec:method-cost}:
$\mathcal{O}(rLkm\cdot \mathrm{cost}_{\mathrm{JVP}})$ with batched JVPs. We report end-to-end wall-clock per experiment and per figure, along with GPU type and batch sizes. 

\subsection{Success criteria (targets)}
\label{sec:targets}
We pre-register the following targets: ROC AUC $\ge 0.75$ and AP $\ge 0.60$ for curvature predicting invariance failures (E6); Spearman $\rho(\Delta,\delta)\ge 0.65$ (E3); area-under-drift decreases with depth stabilization interventions (E4); probe-variance ratio $\le 0.6$ and higher Kendall-$\tau$ after gauge-fix (E2); total overhead $\le 20\%$ at defaults (E7). Deviations from targets will be reported with CIs and effect sizes.

\paragraph{Disclosure.}
All illustrative curves in figures (e.g., Fig.~\ref{fig:diagnostic-suite}) are placeholders; they are replaced by measured values upon running this plan.

\section{Conclusion and Limitations}
\label{sec:conclusion}

We proposed a minimal, \emph{post-hoc} geometric lens for Transformers that yields actionable signals for reliability and performance. The core ingredients are: (i) an \textbf{inverse-free} holonomy score $\kappa_{\mathrm{inv}}$ computed with JVPs (Eq.~\eqref{eq:invfree-kappa}) on a discrete bundle over (position, layer), (ii) \textbf{activation commutators} $\Delta_{A,B}$ to quantify order sensitivity, and (iii) a light \textbf{orthogonal gauge-fix} (whiten\,$+$\,Procrustes) to stabilize analysis and logging. We specified data contracts and a preregistered evaluation plan (E1--E7) with metrics in \S\ref{sec:metrics} and procedures in \S\ref{sec:experiments}. The intended impact is twofold: predict when invariances will break and indicate when fusions/reorderings are safe—under tight runtime budgets.

\subsection{Limitations.}
\label{sec:limits}
\begin{itemize}[leftmargin=1.2em,itemsep=0pt]
\item \textbf{Analogy scope.} The attention-as-connection view is an \emph{operational analogy} on discrete transports. We do not claim a continuous principal-bundle model of full networks nor that linguistic structure is represented as exact group actions.
\item \textbf{Approximation error.} The \emph{frozen-softmax} scan ignores cross-head coupling and context dependence; we confirm only hotspots with full JVP loops. Edge-wise Jacobians and local $1{\times}1$ rectangles approximate broader interactions.
\item \textbf{Coverage \& predictiveness.} Sampling $(r,k,m)$ trades sensitivity for cost. Some failures arise from long-range interactions outside sampled loops; conversely, high curvature need not always surface as user-visible error without the right stimulus.
\item \textbf{Compute overhead.} Although matrix-free, complexity scales as $\mathcal{O}(rLkm\cdot\mathrm{cost}_{\mathrm{JVP}})$ (\S\ref{sec:method-cost}). Budgets and defaults (Table~\ref{tab:exp-setup}) are crucial to keep total overhead $\le 10$--$20\%$.
\item \textbf{External validity.} Our methods target decoder-only Transformers with pre-LN residual blocks and RoPE/relative positions. Extensions to encoder–decoder models, structured caches, and mixture-of-experts are left to future work.
\item \textbf{Threats to validity (planned experiments).} IR labels depend on orbit construction and pass@k choices; distribution shift between calibration batches and deployment inputs can affect $\Delta$–$\delta$ correlations and $\kappa$ calibration.
\item \textbf{Gauge-invariance proof.} A prior proof attempt for gauge invariance had a flaw. We therefore report empirical gauge-behavior checks and unit tests, and leave a formal proof to future work.
\end{itemize}

\subsection{Future work.}
\begin{enumerate}[leftmargin=1.2em,itemsep=0pt]
\item \textbf{Execute preregistered E1--E7} across multiple open models and report measured AUC/AP, $\rho(\Delta,\delta)$, drift areas, variance ratios, and overhead with CIs.
\item \textbf{Sharper theory.} Tighten bounds for Hutchinson/JVP estimation error; study loop-size dependence and concentration; explore alternative gauges (block-orthogonal, subspace-restricted).
\item \textbf{Broader invariances.} Beyond code $\alpha$-renaming: canonical algebraic rewrites, unit conversions, symbolic equalities, and templated paraphrases with ground-truth equivalence certificates.
\item \textbf{Efficiency.} Low-rank/Jacobian-sketch variants; streaming probes; adaptive loop selection (active sampling from attention sparsity or saliency).
\item \textbf{LLM-assisted testing.} Use LLMs to propose invariance orbits, generate minimal counterexamples, and triage hotspots identified by $\kappa_{\mathrm{inv}}$ and $\Delta$.
\item \textbf{Visualization \& observability.} Curvature/commutator dashboards with gauge-stable traces; CI gates that fail on reproducibility regressions and miscalibrated thresholds.
\item \textbf{Architectural reach.} Apply the diagnostics to encoder–decoder, state-space models, and retrieval-augmented pipelines; study how memory modules alter holonomy patterns.
\item \textbf{Integration with semantic anchoring frameworks.}
The Unified Cognitive Consciousness Theory (UCCT) models LLMs as pattern repositories activated by semantic anchors such as few-shot prompts, RAG, and fine-tuning. We hypothesize that WILSON diagnostics can improve anchoring quality by:
\begin{enumerate}[label=(\roman*), leftmargin=1.4em, itemsep=0pt]
  \item \textbf{Anchor placement:} prefer low-$\kappa_{\mathrm{inv}}$ regions where pattern representations are stable and order-insensitive, which may improve few-shot success rates;
  \item \textbf{Cluster quality filtering:} use curvature as a stability score for UCCT pattern clusters. High-$\kappa$ regions may contain fragile or order-sensitive patterns that are unreliable anchoring targets;
  \item \textbf{Orbit construction:} use $\kappa_{\mathrm{inv}}$ maps to predict which semantic variations (for example, $\alpha$-renamings and paraphrases) will break invariance, allowing UCCT to focus testing on high-risk variants and trust low-$\kappa$ transformations;
  \item \textbf{Cross-method coherence:} WILSON and UCCT capture different aspects of stability, computational versus semantic. Testing whether $\kappa_{\mathrm{inv}}$ correlates with UCCT anchoring strength $S$ or an invariance ratio $IR$ could reveal whether geometric stability predicts semantic robustness.
\end{enumerate}
These hypotheses require empirical validation. A shared protocol would compute both WILSON and UCCT metrics on the same tasks (for example, base arithmetic and code synthesis) and measure correlations. If $\kappa_{\mathrm{inv}}$ significantly predicts UCCT anchoring success ($R^2 > 0.5$), the integration could guide prompt engineering and model selection in practice. We leave this to future work, noting that WILSON's gauge-stable logging (\S4.3) already addresses reproducibility challenges that UCCT encounters in cross-seed comparisons.
\end{enumerate}

\paragraph{Takeaway.}
Inverse-free holonomy and activation commutators provide \emph{localized, gauge-stable} signals that are cheap enough to deploy and strong enough to guide testing and safe optimizations. We view these tools as a practical substrate on which more principled invariance testing and reliability engineering for LLMs can be built.

\section*{Acknowledgment of AI Assistance}

The author(s) conceived the problem and core methods (discrete bundle over position and layer, inverse-free Wilson-loop curvature $\kappa_{\mathrm{inv}}$, and commutator-based order-sensitivity diagnostics), designed the E1--E7 suite and metrics, set scope and systems integration, curated related work, and wrote and edited the manuscript.

OpenAI's GPT-5 Thinking was used under direction as a writing and tooling aid to help with restructuring, section polish, and small illustrative code snippets (JVP/Hutchinson examples, commutator gates, planner sketch). The model is not an author. All technical decisions and claims were made and approved by the human author(s); no proprietary or sensitive data were provided; any errors remain the responsibility of the human author(s).

\bibliographystyle{plainnat}
\bibliography{WILSON,Physics,EdwardChang}

\clearpage
\section*{Appendices}
\appendix
\section{Theory Appendix}
\label{app:theory}

\subsection{Scope and what is new}
Results below formalize properties used throughout the paper. 
\emph{What is classical:} orthogonal similarity preserves the Frobenius norm and Rademacher/Gaussian probe distributions \citep[][Sec.~2]{hornjohnson2013matrixanalysis}; Hutchinson-type trace/energy estimators are unbiased with well-studied concentration \citep{hutchinson1989stochastic,avron2011randomized,hsu2012tail,rudelson2013hansonwright}; the equivalence between a zero loop commutator and identity holonomy for \emph{invertible} transports is standard in gauge theory \citep{ambrose1953holonomy,wilson1974confinement}. 
\emph{What is new here:} the \textbf{inverse-free, JVP-only curvature} $\kappa_{\mathrm{inv}}$ tied to \emph{Transformer residual streams} and the \emph{diagnostic role} connecting curvature/commutators to invariance breaks and order sensitivity.

\subsection{Notation and setup}
Let \(B=\{1,\dots,T\}\times\{0,\dots,L\}\) be the product base (Position\(\times\)Layer).
At node \((i,\ell)\) the residual stream is a vector in \(\mathbb{R}^d\).
Vertical transports \(T^{\mathrm{layer}}_{i,\ell}:\mathbb{R}^d\!\to\!\mathbb{R}^d\) map \((i,\ell)\!\to\!(i,\ell{+}1)\);
horizontal transports \(T^{\mathrm{attn}}_{i\leftarrow j,\ell}:\mathbb{R}^d\!\to\!\mathbb{R}^d\) map \((j,\ell)\!\to\!(i,\ell)\).
Define
\[
A_{i,j,\ell}:=T^{\mathrm{layer}}_{i,\ell}T^{\mathrm{attn}}_{i\leftarrow j,\ell}
- T^{\mathrm{attn}}_{i\leftarrow j,\ell+1}T^{\mathrm{layer}}_{j,\ell}\,:\ \mathbb{R}^d\to\mathbb{R}^d.
\]
The inverse-free curvature (Def.~\eqref{eq:invfree-kappa}) is
\[
\kappa_{\mathrm{inv}}(i,j,\ell)^2 \;=\; \mathbb{E}_{v}\,\big\|A_{i,j,\ell}\,v\big\|_2^2
\;=\; \mathrm{tr}\!\big(A_{i,j,\ell}^\top A_{i,j,\ell}\big)
\;=\;\|A_{i,j,\ell}\|_F^2,
\]
with expectation over Rademacher or standard Gaussian probes \(v\in\mathbb{R}^d\)
(\(\mathbb{E}[vv^\top]=I\)).

\subsection{Hutchinson/JVP estimator: unbiasedness and concentration}
\label{app:hutchinson}
Let \(A:=A_{i,j,\ell}\), \(\theta:=\|A\|_F^2\), and \(v_s\overset{\text{i.i.d.}}{\sim}\{\pm1\}^d\).
Define
\[
\widehat{\theta}_r \;=\; \frac{1}{r}\sum_{s=1}^r \|A v_s\|_2^2
= \frac{1}{r}\sum_{s=1}^r v_s^\top (A^\top A)\,v_s.
\]

\begin{lemma}[Unbiasedness \citep{hutchinson1989stochastic,avron2011randomized}]
\label{lem:unbiased}
\(\mathbb{E}\,\widehat{\theta}_r = \theta\).
\end{lemma}

\begin{proof}
\(\mathbb{E}\|Av\|_2^2=\mathbb{E}\mathrm{tr}(v^\top A^\top A v)=\mathrm{tr}(A^\top A)=\theta\), using \(\mathbb{E}[vv^\top]=I\).
\end{proof}

\begin{theorem}[Concentration via Hanson--Wright \citep{hsu2012tail,rudelson2013hansonwright}]
\label{thm:concentration}
Let \(\mathrm{sr}(A):=\|A\|_F^2/\|A\|_2^2\) be the stable rank. For any \(\varepsilon\in(0,1)\),
\[
\Pr\!\left(\left|\widehat{\theta}_r-\theta\right|\ge \varepsilon\,\theta\right)
\;\le\; 2\exp\!\left(-c_1\,r\,\min\big\{\varepsilon^2,\;\varepsilon\,\mathrm{sr}(A)\big\}\right),
\]
for universal constant \(c_1>0\). Thus it suffices that
\(r \gtrsim \min\{\varepsilon^{-2},(\varepsilon\,\mathrm{sr}(A))^{-1}\}\log(2/\delta)\)
to get error \(\le \varepsilon\theta\) with prob.\ \(\ge 1-\delta\).
\end{theorem}

\begin{proof}[Proof sketch]
Each \(Z_s=\|Av_s\|_2^2=v_s^\top (A^\top A)v_s\) is a sub-Gaussian quadratic form; Hanson--Wright controls tails. Averaging \(r\) i.i.d.\ copies gives the stated rate; use \(\|A^\top A\|_F^2=\|A\|_F^4\), \(\|A^\top A\|_2=\|A\|_2^2\) to expose \(\mathrm{sr}(A)\).
\end{proof}

\begin{corollary}[RMS rate]
\label{cor:mse}
\(\mathrm{MSE}(\widehat{\theta}_r)=\mathcal{O}(\|A\|_F^4/r)\); by the delta method, \(\mathrm{SE}(\kappa_{\mathrm{inv}})\approx \tfrac{1}{2}\theta^{-1/2}\sqrt{\mathrm{Var}[\widehat{\theta}_r]}\).
\end{corollary}

\subsection{Work and memory complexity under \texorpdfstring{\((r,k,m)\)}{(r,k,m)} sampling}
\label{app:complexity}
Per layer, sample \(k\) targets \(i\) and top-\(m\) neighbors \(j\) (by attention mass). 
\textbf{Scan (frozen-softmax):}
\(\text{work}=\mathcal{O}(Lkm\cdot H d_h^2)\).
\textbf{Confirm (JVP):} two JVP paths per probe \(\Rightarrow\) \(4rLkm\) JVPs total:
\(\text{work}=\mathcal{O}(rLkm\cdot \mathrm{cost}_{\mathrm{JVP}})\).
Only hotspots (few\%) need confirmation in practice.

\subsection{Connections and limiting cases}
\label{app:connections}
\paragraph{Linear/tied case.}
If sublayers are linear and \(T^{\mathrm{attn}}_{\ell+1}\!=\!T^{\mathrm{attn}}_{\ell}\), then
\(A=[T^{\mathrm{layer}},T^{\mathrm{attn}}]\) and \(\kappa_{\mathrm{inv}}^2=\|[T^{\mathrm{layer}},T^{\mathrm{attn}}]\|_F^2\).

\subsection{Assumptions and caveats}
\label{app:caveats}
(1) Transports are Jacobian blocks (or frozen-softmax approximations) on the residual stream; non-invertibilities motivate inverse-free loops. (2) Concentration uses sub-Gaussian probes (Rademacher/Gaussian). (3) The stable-rank dependence captures conditioning; ill-conditioned loops require larger \(r\).   
\section{Code Appendix}
\label{app:code}

\lstdefinestyle{apxpy}{
  language=Python,
  basicstyle=\ttfamily\footnotesize,
  breaklines=true,
  breakatwhitespace=true,
  columns=fullflexible,
  keepspaces=true,
  postbreak=\mbox{\textcolor{gray}{$\hookrightarrow$}\space}
}
\lstset{style=apxpy}
\UseRawInputEncoding
\begin{lstlisting}[caption={Inverse-free curvature with JVPs (matrix-free). Requires PyTorch 2.0+ (`torch.func.jvp`).}]
import torch
from torch.func import jvp

# Each fn_* maps residual stream -> residual stream (same shape).
# attn_l:   f(h) = within-layer attention at layer ℓ (output at position i)
# attn_lp1: f(h) = within-layer attention at layer ℓ+1 (output at position i)
# layer_i:  f(h) = vertical transport (i,ℓ) -> (i,ℓ+1)
# layer_j:  f(h) = vertical transport (j,ℓ) -> (j,ℓ+1)
# h_l:      residual stream at layer ℓ (batch or single vector at position i/j)
def kappa_inv_once(attn_l, attn_lp1, layer_i, layer_j, h_l, v):
    # Path A: attn@ℓ then layer@i
    _, j_attn_v = jvp(attn_l, (h_l,), (v,))
    _, pathA    = jvp(layer_i, (h_l,), (j_attn_v,))
    # Path B: layer@j then attn@ℓ+1
    _, j_layer_v = jvp(layer_j, (h_l,), (v,))
    _, pathB     = jvp(attn_lp1, (h_l,), (j_layer_v,))
    return torch.norm(pathA - pathB)

def rademacher_like(x):
    # Returns ±1 with same dtype/device; supports float/bfloat16.
    return (torch.randint_like(x, low=0, high=2).mul_(2).sub_(1)) .to(dtype=x.dtype)

def kappa_inv(attn_l, attn_lp1, layer_i, layer_j, h_l, r=6):
    # Hutchinson average of squared norms, then sqrt.
    d = h_l.shape[-1]
    acc = h_l.new_zeros(())
    for _ in range(r):
        v = rademacher_like(h_l[..., :d])  # shape-compatible probe
        acc = acc + kappa_inv_once(attn_l, attn_lp1, layer_i, layer_j, h_l, v)**2
    return (acc / r).sqrt()
\end{lstlisting}

\begin{lstlisting}[caption={Frozen-softmax horizontal transport (per-head) as a fast scan surrogate.}]
# Builds a block-diagonal approximation ∑_h α_ij^(h) W_V^(h) W_O^(h)
# Use for scanning; confirm hotspots with JVP-based kappa_inv().
def frozen_attn_transport(alpha_ij_per_head, W_V_list, W_O_list):
    # alpha_ij_per_head: list[H] of scalars (on a calibration batch)
    # W_V_list, W_O_list: list[H] of (d_h x d_h) tensors
    H = len(W_V_list); d_h = W_V_list[0].shape[0]
    T = W_V_list[0].new_zeros((H*d_h, H*d_h))
    off = 0
    for h in range(H):
        T[off:off+d_h, off:off+d_h] = alpha_ij_per_head[h] * (W_V_list[h] @ W_O_list[h])
        off += d_h
    return T
\end{lstlisting}

\begin{lstlisting}[caption={Commutator map for two submodules A,B on a calibration batch.}]
# Computes Δ_{A,B} = || A(B(X)) - B(A(X)) ||_F over a batch X.
@torch.no_grad()
def commutator_norm(A, B, X):
    AB = A(B(X))
    BA = B(A(X))
    return torch.linalg.vector_norm(AB - BA)
\end{lstlisting}

\begin{lstlisting}[caption={Gauge-fix: whitening + orthogonal Procrustes alignment.}]
# Whitening: H Σ^{-1/2}, where Σ is covariance over tokens/batch dims.
def whiten(H, eps=1e-5):
    # H: (N, d) matrix of features (stack tokens/batch as N)
    mu = H.mean(dim=0, keepdim=True)
    X  = H - mu
    # Covariance and eigen
    C = (X.T @ X) / max(1, H.shape[0] - 1)
    evals, evecs = torch.linalg.eigh(C)
    Dm12 = torch.diag(torch.clamp(evals, min=eps).rsqrt())
    W = evecs @ Dm12 @ evecs.T
    return (X @ W), (mu, W)

# Orthogonal Procrustes: align H2 to H1 (minimize ||H1 - H2 R||_F over R in O(d)).
def procrustes_R(H1, H2):
    # assumes H1,H2 are whitened; returns R (d x d)
    U, _, Vh = torch.linalg.svd(H2.T @ H1, full_matrices=False)
    return U @ Vh

def gauge_fix(H1, H2):
    H1w, _ = whiten(H1)
    H2w, _ = whiten(H2)
    R = procrustes_R(H1w, H2w)
    return H1w, (H2w @ R), R
\end{lstlisting}

\begin{lstlisting}[caption={CSV schema helpers and writers for holonomy/commutator outputs.}]
import csv

def write_holonomy_csv(path, rows):
    # rows: iterable of dicts with keys ["position","layer","kappa"]
    with open(path, "w", newline="") as f:
        w = csv.DictWriter(f, fieldnames=["position","layer","kappa"])
        w.writeheader()
        for r in rows: w.writerow(r)

def write_commutator_csv(path, rows):
    # rows: iterable of dicts with keys ["i","j","value"]
    with open(path, "w", newline="") as f:
        w = csv.DictWriter(f, fieldnames=["i","j","value"])
        w.writeheader()
        for r in rows: w.writerow(r)
\end{lstlisting}

\begin{lstlisting}[caption={Policy: SafeFuse gate using commutator and drift surrogate.}]
def plan_step(graph, kappa_map, invariants, tau_kappa=0.12):
    # Partition frontier by curvature threshold
    par = [u for u in graph.frontier() if kappa_map[u] <= tau_kappa]
    seq = [u for u in graph.frontier() if kappa_map[u] >  tau_kappa]
    # Low-risk: parallel with invariant guards
    exec_parallel([guarded(task, invariants) for task in par])
    # High-risk: sequential + extra verification
    for task in seq:
        run_with_checks(task, invariants, extra_verifiers=True)
\end{lstlisting}

\begin{lstlisting}[caption={Colab harness sketch for E1--E7 (seed, knobs, figure/CSV export).}]
def run_suite(model, tokenizer, cfg):
    torch.manual_seed(cfg.seed)
    # E1: alpha-renaming invariance
    ir = run_alpha_renaming(model, tokenizer, cfg.alpha_suite)
    # E2/E7: gauge-fix stability
    stab = run_gauge_stability(model, cfg.gauge_suite)
    # E3: commutator heatmap
    comm = run_commutator_maps(model, cfg.comm_suite)
    write_commutator_csv(cfg.out_dir /"commutator.csv", comm.rows)
    # E4: RoPE drift
    drift = run_rope_drift(model, cfg.rope_suite)
    # E5/E6: curvature maps + prediction
    holo = run_holonomy_maps(model, cfg.holo_suite)   # writes holonomy.csv
    roc  = run_curvature_predicts_fail(model, cfg.pred_suite)
    # Export figures/tables
    save_figs(cfg.out_dir); save_tables(cfg.out_dir)
    return {"IR": ir, "stability": stab, "drift": drift, "roc": roc}
\end{lstlisting}


\balance
\section{Minimal black-box demo in ChatGPT}
\label{app:blackbox-demo}

This appendix shows how to approximate WILSON-style checks in a public chat interface with no access to logits, JVPs, or internals. We use black-box proxies:
\begin{itemize}[leftmargin=1.2em,itemsep=2pt]
\item \textbf{Invariance ratio (IR)}: fraction of paraphrases that yield the same final answer as a majority reference for a fixed base query.
\item \textbf{Pathway discrepancy ratio (PDR)}: fraction of pathway templates (for example, two styles of prompting) that yield different final answers on the same query.
\item \textbf{Ordering drift (OD)}: indicator that swapping two context passages changes the final answer.
\end{itemize}
These proxies cannot replace $\kappa_{\mathrm{inv}}$ or commutators. They demonstrate the operational idea: measure, threshold, and then act.

\subsection{Setup and scoring}
\label{app:setup}
Use a single chat thread. For each run, request a \emph{final answer only} format to reduce variance.
\begin{quote}
\small
\texttt{Reply with the final answer only. If multiple choices are shown, reply with the letter only.}
\end{quote}
For each task, record outputs in a CSV with the columns shown in Listing~\ref{lst:csv-template}. Compute:
\begin{footnotesize}
\[
\mathrm{IR}(q)=\frac{\text{\# paraphrases that match the majority answer}}{\text{total paraphrases}},
\]
\[
\mathrm{PDR}(q)=\frac{\text{\# pathway prompts with different final answers}}{\text{total pathways}},\]
\[
\mathrm{OD}(q)=\mathbb{1}[\text{answer changes under passage swap}].
\]
\end{footnotesize}

\begin{lstlisting}[
  language={}, label={lst:csv-template},
  basicstyle=\ttfamily\footnotesize,
  caption={CSV template for black-box scoring.}
]
task_id,variant,condition,input_id,
final_answer,correct,notes
paraphrase,base,MCQ,v0,A,1,""
paraphrase,paraphrase,MCQ,v1,A,1,""
paraphrase,paraphrase,MCQ,v2,B,0,"changed"
pathway,templateA,MCQ,tA,A,1,"answer only"
pathway,templateB,MCQ,tB,A,1,"think internally, answer only"
ordering,A_then_B,context,q1,Tuesday,1,""
ordering,B_then_A,context,q1,Monday,0,"drift"
\end{lstlisting}

\subsection{Demo 1: Prompt robustness under paraphrase (IR)}
\label{app:demo-paraphrase}
We use a simple multiple-choice item with a known answer. The goal is not to trick the model, but to measure stability under harmless rewordings.

\paragraph{Base item (copy into the chat).}
\begin{quote}\small
Which is larger?  
A) $2^{100}$  
B) $10^{30}$  
C) They are equal  
D) Not enough information  
Reply with the letter only.
\end{quote}

\paragraph{Paraphrase set.} Ask the same question six ways. Paste one at a time, each in a fresh message in the same thread.
\begin{enumerate}[leftmargin=1.2em,itemsep=1pt]
\item Which quantity has the greater magnitude, $2^{100}$ or $10^{30}$? Choose A, B, C, or D.
\item Compare $2^{100}$ and $10^{30}$. Which is bigger? A, B, C, or D.
\item Decide which value exceeds the other: $2^{100}$ vs $10^{30}$. Answer with A, B, C, or D.
\item Select the larger value from the options: A) $2^{100}$, B) $10^{30}$, C) equal, D) insufficient data.
\item Among $2^{100}$ and $10^{30}$, which is numerically greater? Reply A, B, C, or D.
\item Choose the correct option that identifies the larger number: A) $2^{100}$, B) $10^{30}$, C) equal, D) not enough information.
\end{enumerate}

\paragraph{Scoring.} Let the majority answer across the six paraphrases define the reference. Compute $\mathrm{IR}$ as the fraction that match the reference. Record in the CSV. A robust prompt shows $\mathrm{IR}=1$.

\subsection{Demo 2: Pathway stability with template prompts (PDR)}
\label{app:demo-pathway}
Use the same base item as in A.2. Issue two prompts that differ only in pathway style.

\paragraph{Template A.}
\begin{quote}\small
Answer with the final letter only. Do not include steps.
\end{quote}

\paragraph{Template B.}
\begin{quote}\small
Think through the problem internally. Output only the final letter. Do not include steps.
\end{quote}

\paragraph{Scoring.} If the letters differ, count a discrepancy. Compute $\mathrm{PDR}$ over a small set of items. High $\mathrm{PDR}$ suggests fragile reasoning pathways.

\subsection{Demo 3: Simulated passage ordering sensitivity (OD)}
\label{app:demo-ordering}
We simulate RAG ordering with two short synthetic passages.

\paragraph{Passage A.}
\begin{quote}\small
Alpha City Library moved its weekly closure from Monday to Tuesday in July. Notices were posted on June 28.
\end{quote}

\paragraph{Passage B.}
\begin{quote}\small
The library used to close on Mondays to balance weekend staffing. Some older flyers still mention Monday.
\end{quote}

\paragraph{Query.}
\begin{quote}\small
On which weekday is the library currently closed?
\end{quote}

\paragraph{Prompt 1: A then B.}
\begin{quote}\small
Context: [A][B]  
Question: On which weekday is the library currently closed?  
Reply with a single weekday.
\end{quote}

\paragraph{Prompt 2: B then A.}
\begin{quote}\small
Context: [B][A]  
Question: On which weekday is the library currently closed?  
Reply with a single weekday.
\end{quote}

\paragraph{Scoring.} If answers differ, set $\mathrm{OD}=1$ for this item. This approximates position bias in retrieval order.

\subsection{Two-model correlation without internals}
\label{app:twomodel}
Repeat A.2 to A.4 on two models, denoted A and B, using the same paraphrases, templates, and contexts.

\paragraph{Per item measures.} For each base item $q$:
\[
\mathrm{IR}_A(q),\ \mathrm{IR}_B(q),\ \mathrm{PDR}_A(q),\ \mathrm{PDR}_B(q),\ \mathrm{OD}_A(q),\ \mathrm{OD}_B(q).
\]
Define a sensitivity index $\mathrm{SI}(q)=1-\mathrm{IR}(q)$. Define cross-model drift $\mathrm{D}_{A,B}(q)=\mathbb{1}[\text{majority answer of A} \neq \text{majority answer of B}]$.

\paragraph{Analysis.} Compute Pearson $r$ or Spearman $\rho$ between $\mathrm{SI}_A$ and $\mathrm{D}_{A,B}$ across items. Repeat for $\mathrm{SI}_B$. If $\mathrm{SI}$ correlates with cross-model disagreement, then paraphrase sensitivity predicts when models diverge. Report the mean of $\mathrm{OD}$ as an ordering risk rate.

\subsection{Relation to WILSON signals}
\label{app:relation}
The black-box proxies approximate WILSON’s intent.
\begin{itemize}[leftmargin=1.2em,itemsep=1pt]
\item High $\mathrm{SI}$ suggests regions where WILSON would report high $\kappa_{\mathrm{inv}}$ for the relevant positions and layers.
\item Nonzero $\mathrm{OD}$ suggests a nonzero commutator between the two ordering paths in the attention or residual pipeline.
\item A two-model study that correlates $\mathrm{SI}$ with $\mathrm{D}_{A,B}$ mirrors a cross-model analysis of curvature clusters under a gauge-stable alignment, which we reserve for future work.
\end{itemize}

\subsection{Copy-ready prompt bundle}
\label{app:prompt-bundle}
Paste each block as a separate message.

\paragraph{Header.}
\begin{quote}\small
Reply with the final answer only. If multiple choices are shown, reply with the letter only.
\end{quote}

\paragraph{Paraphrase set.}
\begin{quote}\small
[Q0] Which is larger? A) $2^{100}$ B) $10^{30}$ C) equal D) not enough information.  
[Q1] Which quantity has the greater magnitude, $2^{100}$ or $10^{30}$? Choose A, B, C, or D.  
[Q2] Compare $2^{100}$ and $10^{30}$. Which is bigger? A, B, C, or D.  
[Q3] Decide which value exceeds the other: $2^{100}$ vs $10^{30}$. Answer with A, B, C, or D.  
[Q4] Select the larger value from the options: A) $2^{100}$ B) $10^{30}$ C) equal D) insufficient data.  
[Q5] Among $2^{100}$ and $10^{30}$, which is numerically greater? Reply A, B, C, or D.  
[Q6] Choose the correct option that identifies the larger number: A) $2^{100}$ B) $10^{30}$ C) equal D) not enough information.
\end{quote}

\paragraph{Pathway templates.}
\begin{quote}\small
Template A: Answer with the final letter only.  
Template B: Think through the problem internally. Output only the final letter. Do not include steps.
\end{quote}

\paragraph{Ordering prompts.}
\begin{quote}\small
A then B: Context: [Alpha City Library moved its weekly closure from Monday to Tuesday in July. Notices were posted on June 28.] [The library used to close on Mondays to balance weekend staffing. Some older flyers still mention Monday.] Question: On which weekday is the library currently closed? Reply with a single weekday.

B then A: Context: [The library used to close on Mondays to balance weekend staffing. Some older flyers still mention Monday.] [Alpha City Library moved its weekly closure from Monday to Tuesday in July. Notices were posted on June 28.] Question: On which weekday is the library currently closed? Reply with a single weekday.
\end{quote}

\subsection{What this demo does not measure}
\label{app:limits}
This procedure does not estimate $\kappa_{\mathrm{inv}}$, commutators, or any internal transport. It uses agreement rates and ordering flips as observable stand-ins. The purpose is to give practitioners a taste of the WILSON workflow that fits public chat interfaces.

\subsection{Optional: full WILSON cross-model plan}
\label{app:full-crossmodel}
With internal hooks, one can correlate two models A and B as follows: (i) compute curvature maps for a shared probe set, (ii) apply a gauge-stable alignment by whitening and Procrustes on matched fibers, (iii) compare curvature clusters and predict high-risk items where both models have elevated curvature near the same layers and positions, (iv) validate with black-box drift on those items. We leave this to future work.

\end{document}